\documentclass{article}

\usepackage{xcolor}
\usepackage{microtype}
\usepackage{graphicx}
\usepackage{subfigure}
\usepackage{import, ifthen}

\usepackage[accepted]{icml2021modified}

\usepackage[pagebackref]{hyperref}       
\hypersetup{
  colorlinks=true,
  linkcolor=blue,
  filecolor=magenta,
  urlcolor=cyan,
  citecolor=blue
}
\usepackage{my_macros}
\usepackage{amsmath,amssymb,mathtools,enumitem,hyphenat,float}

\newcommand{\squeeze}{\textstyle } 

\newcommand{\Lave}{\bar{L}}
\newcommand{\prm}[1]{\pi_{#1}}

\def\<#1,#2>{\langle #1,#2\rangle}

\newcommand{\sigmaesc}{\sigma_{\ast}}

\def\xx{{\boldsymbol x}}
\def\yy{{\boldsymbol y}}
\newcommand{\sigmarr}{\sigma_{\mathrm{rad}}}
\newcommand{\newqedhere}{\tag*{\qedhere}}

\usepackage{url}            
\usepackage{booktabs}       
\usepackage{amsfonts}       
\usepackage{nicefrac}       
\usepackage{microtype}      
\usepackage{enumitem}
\usepackage{cleveref}
\theoremstyle{definition}
\newtheorem{Definition}{Definition} 
\newtheorem{Example}{Example} 
\newtheorem{Assumption}{Assumption}
\newtheorem{Lemma}{Lemma}
\newtheorem{Theorem}{Theorem}
\crefname{Lemma}{Lemma}{lemmas}
\crefname{Proposition}{Proposition}{Propositions}
\crefname{Theorem}{theorem}{theorems}
\crefname{Assumption}{Assumption}{Assumptions}
\Crefname{Lemma}{Lemma}{Lemmas}
\Crefname{Proposition}{Proposition}{Propositions}
\Crefname{Theorem}{Theorem}{Theorems}
\Crefname{Assumption}{Assumption}{Assumptions}
\usepackage{multirow}
\usepackage{threeparttable,adjustbox}
\definecolor{mygreen}{RGB}{77,175,74}
\definecolor{myblue}{RGB}{55,126,184}
\definecolor{mydarkgreen}{RGB}{39,130,67}

\usepackage[noend]{algpseudocode}

\newcommand{\rrbox}[1]{\colorbox{myblue!30}{#1}}
\newcommand{\sobox}[1]{\colorbox{mygreen!30}{#1}}

\icmltitlerunning{Proximal and Federated Random Reshuffling}

\begin{document}

\twocolumn[
\icmltitle{Proximal and Federated Random Reshuffling}



\icmlsetsymbol{equal}{*}

\begin{icmlauthorlist}
\icmlauthor{Konstantin Mishchenko}{kaust}
\icmlauthor{Ahmed Khaled}{kaust}
\icmlauthor{Peter Richt\'{a}rik}{kaust}
\end{icmlauthorlist}

\icmlaffiliation{kaust}{King Abdullah University of Science and Technology, Thuwal, Saudi Arabia}

\icmlcorrespondingauthor{Konstantin Mishchenko}{konsta.mish@gmail.com}

\icmlkeywords{Machine Learning, ICML}

\vskip 0.3in
]



\printAffiliationsAndNotice{}  

\begin{abstract}
    Random Reshuffling (RR), also known as Stochastic Gradient Descent (SGD) without replacement, is a popular and theoretically grounded method for finite-sum minimization. We propose two new algorithms: Proximal and Federated Random Reshuffing (ProxRR and FedRR). The first algorithm, ProxRR, solves composite convex finite-sum minimization problems in which the objective is the sum of a (potentially non-smooth) convex regularizer and an average of $n$ smooth objectives. We obtain the second algorithm, FedRR, as a special case of ProxRR applied to a reformulation of distributed problems with either homogeneous or heterogeneous data. We study the algorithms' convergence properties with constant and decreasing stepsizes, and show that they have considerable advantages over Proximal and Local SGD. In particular, our methods have superior complexities and ProxRR evaluates the proximal operator once per epoch only. When the proximal operator is expensive to compute, this small difference makes ProxRR up to $n$ times faster than algorithms that evaluate the proximal operator in every iteration. We give examples of practical optimization tasks where the proximal operator is difficult to compute and ProxRR has a clear advantage. Finally, we corroborate our results with experiments on real data sets.
\end{abstract}

\section{Introduction}

Modern theory and  practice of training supervised machine learning models is based on the paradigm of regularized empirical risk minimization (ERM)~\citep{shai_book}. While the ultimate goal of supervised learning is  to train models that generalize well to unseen data, in practice only a finite data set is available during training. Settling for a model merely minimizing the average loss on this training set---the empirical risk---is insufficient, as this often leads to over-fitting and poor generalization performance in practice. Due to this reason, empirical risk is virtually always amended with a suitably chosen regularizer whose role is to encode prior knowledge about the learning task at hand, thus biasing the training algorithm towards better performing models.

The regularization framework is quite general and perhaps surprisingly it also allows us to consider methods for federated learning (FL)---a paradigm in which we aim at training model for a number of clients that do not want to reveal their data~\citep{FEDLEARN, FL2017-AISTATS, FL-big}. The training in FL usually happens on devices with only a small number of model updates being shared with a global host. To this end, Federated Averaging algorithm has emerged that performs Local SGD updates on the clients' devices and periodically aggregates their average. Its analysis usually requires special techniques and deliberately constructed sequences hindering the research in this direction. We shall see, however, that the convergence of our FedRR follows from merely applying our algorithm for regularized problems to a carefully chosen reformulation.

Formally, regularized ERM problems are optimization problems of the form \begin{equation}
  \label{eq:finite-sum-min}
  \squeeze
  \min \limits_{x \in \R^d} \Bigl [ P(x) \eqdef \frac{1}{n} \sum \limits_{i=1}^{n} f_{i} (x)  + \psi(x)\Bigr ],
\end{equation}
where $f_i:\R^d \to \R$ is the loss of model parameterized by vector $x\in \R^d$ on the $i$-th training data point, and $\psi:\R^d\to \R \cup \{+\infty\}$ is a regularizer.  Let $[n]\eqdef \{1,2,\dots,n\}$. We shall make the following assumption throughout the paper without explicitly mentioning it:

\begin{Assumption}\label{as:smooth_fi_proper_psi}
The functions $f_i$ are $L_i$-smooth and convex, and  the regularizer $\psi$ is proper, closed and convex. Let $L_{\max} \eqdef \max_{i\in [n]} L_i$.
\end{Assumption}

In some results we will additionally assume that either the individual functions $f_i$, or their average $f\eqdef \frac{1}{n}\sum_i f_i$, or the regularizer $\psi$ are $\mu$-strongly convex.  Whenever we need such additional assumptions, we will make this explicitly clear. While all these concepts are standard, we review them briefly in Section~\ref{sec:basic_notions}.

{\bf Proximal SGD.} 
When the number $n$ of training data points  is huge, as is increasingly common in practice, the most efficient algorithms for solving \eqref{eq:finite-sum-min} are stochastic first-order methods, such as stochastic gradient descent (SGD)~\citep{bordes2009sgd}, in one or another of its many variants proposed in the last decade~\citep{Shang2018,pham2020proxsarah}. These method almost invariably rely on alternating stochastic gradient  steps with the evaluation of the proximal operator
\[
\squeeze
	\prox_{\gamma\psi}(x)
	\eqdef \argmin_{z\in \R^d} \left\{\gamma\psi(z) + \frac{1}{2}\|z-x\|^2\right\}.
\]
The simplest of these has the form
\begin{equation}\label{eq:prox-SGD}
x_{k+1}^{\mathrm{SGD}} = \prox_{\gamma_k \psi}(x_k^{\mathrm{SGD}} - \gamma_k \nabla f_{i_k}(x_k^{\mathrm{SGD}})),
\end{equation}
where $i_k$ is an index from $\{1,2,\dots,n\}$ chosen uniformly at random, and $\gamma_k>0$ is a properly chosen learning rate. Our understanding of \eqref{eq:prox-SGD} is quite mature; see \citep{Gorbunov2020} for a general  treatment which considers  methods of this form in conjunction with more advanced stochastic gradient estimators in place of $\nabla f_{i_k}$.

Applications such as training sparse linear models~\citep{tibshirani1996regression}, non-negative matrix factorization~\citep{lee1999learning}, image deblurring~\citep{bredies2010total}, and training with group selection~\citep{yuan2006model} all rely on the use of hand-crafted  regularizes. For most of them, the proximal operator can be evaluated efficiently, and SGD is near or at the top of the list of efficient training algorithms.

{\bf Random reshuffling.} A particularly successful variant of SGD is based on the idea of random shuffling (permutation) of the training data followed by $n$ iterations of the form \eqref{eq:prox-SGD}, with the index $i_k$ following the pre-selected  permutation~\citep{bottou2012stochastic}. This process is repeated several times, each time using a new freshly sampled random permutation of the data, and the resulting  method is known under the name {\em Random Reshuffling (RR)}.\footnote{While we will comment on this in more detail later, RR is not known to converge in the proximal setting, i.e., if $\psi \neq 0$. Moreover, it is not even clear if this is the right proximal extension of RR.} When the same permutation is used throughout, the technique is known under the name {\em Shuffle Once (SO)}.

One of the main advantages of this approach is rooted in its intrinsic ability to avoid cache misses when reading the data from memory, which enables a significantly faster implementation.  Furthermore, RR is often observed to converge in fewer iterations than SGD in practice. This can intuitively be ascribed to the fact that while due to its sampling-with-replacement approach SGD can miss to learn from some data points in any given epoch,  RR will necessarily learn from each data point in each epoch.

Understanding the random reshuffling trick, and why it works, has been a non-trivial open problem for a  long time~\citep{Bottou2009,RR-conjecture2012,Gurbuzbalaban2019RR,haochen2018random}. Until recent development   which lead to a significant simplification of the convergence analysis technique and proofs \citep{MKR2020rr},  prior state of the art relied on long and  elaborate proofs requiring sophisticated arguments and tools, such as  analysis via the Wasserstein distance~\citep{Nagaraj2019}, and relied on a significant number of strong assumptions about the objective \citep{shamir2016without, haochen2018random}. In alternative recent development, \citet{Ahn2020} also develop new tools for analyzing the convergence of random reshuffling, in particular using decreasing stepsizes and for objectives satisfying the Polyak-{\L}ojasiewicz condition, a generalization of strong convexity~\citep{polyak1963gradient, lojasiewicz1963topological}.

The difficulty of analyzing RR has been the main obstacle in the development of even some of the most seemingly benign extensions of the method. Indeed, while all these are well understood in combination with its much simpler-to-analyze cousin SGD,  {\em to the best of our knowledge, there exists no theoretical analysis of proximal, parallel, and importance sampling variants of RR with both constant and decreasing stepsizes, and in most cases it is not even clear how should such methods be constructed.
} Empowered by and building on the recent  advances  of \cite{MKR2020rr}, in this paper we address all these challenges.

\section{Contributions}

In this section we outline the key contributions  of our work, and also offer a few intuitive explanations motivating some of the development.

{\bf From RR to proximal RR.}
Despite rich literature on proximal SGD \citep{Gorbunov2020}, it is not obvious how one should extend RR to solve problem~\eqref{eq:finite-sum-min} when a  nonzero regularizer $\psi$ is present. Indeed, the standard practice for SGD is to apply the proximal operator after each stochastic step~\citep{duchi2009efficient}, i.e., in analogy with \eqref{eq:prox-SGD}. On the other hand, RR is motivated by the fact that a data pass approximates the full gradient step. If we apply the proximal operator after each iteration of RR, we would no longer approximate the full gradient after an epoch, as illustrated by the next example.
\begin{Example}
	Let $n=2$, $\psi(x)=\frac{1}{2}\|x\|^2$, $f_1(x)=\<c_1, x>$, $f_2(x)=\<c_2, x>$ with some $c_1, c_2\in\mathbb{R}^d$, $c_1\neq c_2$. Let $x_0\in\mathbb{R}^d$, $\gamma>0$ and define $x_1 = x_0 - \gamma \nabla f_1(x_0)$, $x_2=x_1-\gamma\nabla f_2(x_1)$. Then, we have $\prox_{2\gamma\psi}(x_2)=\prox_{2\gamma\psi}(x_0-2\gamma\nabla f(x_0))$. However, if $\tilde x_1=\prox_{\gamma\psi}(x_0 - \gamma \nabla f_1(x_0))$ and $\tilde x_2=\prox_{\gamma \psi}(x_1-\gamma \nabla f_2(\tilde x_1))$, then $\tilde x_2\neq \prox_{2\gamma\psi}(x_0-2\gamma\nabla f(x_0))$.
\end{Example}

Motivated by this observation, we propose ProxRR (\Cref{alg:proxrr}), in which the proximal operator is applied at the end of each epoch of RR, i.e., after each pass through all randomly reshuffled data. 

\begin{algorithm}[t]
    \caption{Proximal Random Reshuffling \rrbox{(ProxRR)} and Shuffle-Once \sobox{(ProxSO)}}
    \label{alg:proxrr}
    \begin{algorithmic}[1]
    \Require Stepsizes $\gamma_t > 0$, initial vector $x_0 \in \R^d$, number of epochs $T$
    \State Sample a permutation $\pi =(\prm{0}, \prm{1}, \ldots, \prm{n-1})$ of $[n]$   \sobox{(Do step 1 only for ProxSO)} 
        \For{epochs $t=0,1,\dotsc,T-1$}
            \State  Sample a permut.\ $\pi =(\prm{0}, \prm{1}, \ldots, \prm{n-1})$ of $[n]$  \rrbox{(Do step 3 only for ProxRR)}
            \State $x_t^0 = x_t$
            \For{$i=0, 1, \ldots, n-1$}
                \State $x_{t}^{i+1} = x_t^{i} - \gamma_t \nabla f_{\prm{i}} (x_t^i)$
            \EndFor
        \State $x_{t+1} = \prox_{\gamma_t n \psi}(x_{t}^{n})$
        \EndFor
    \end{algorithmic}
\end{algorithm}

A notable property of \Cref{alg:proxrr} is that {\em only a single proximal operator evaluation is needed during each data pass.} This is in sharp contrast with the way proximal SGD  works, and offers significant advantages in regimes where the evaluation of the proximal mapping is expensive (e.g., comparable to the evaluation of $n$ gradients $\nabla f_1, \dots, \nabla f_n$).

We establish several convergence results for ProxRR, of which we highlight two here. Both offer a linear convergence rate with a fixed stepsize to a neighborhood of the solution. Firstly, in the case when each $f_i$ is $\mu$-strongly convex, we prove the rate (see Theorem~\ref{thm:f-strongly-convex-psi-convex})
\[
\squeeze
			\ecn{x_T - x_\ast} \leq \br{1 - \gamma \mu}^{n T} \sqn{x_0 - x_\ast} + \frac{2 \gamma^2 \sigmarr^2}{\mu},
\]
where $\gamma_t = \gamma \leq \frac{1}{L_{\max}}$ is the stepsize,   and $\sigmarr^2$ is a {\em shuffling radius} constant (for precise definition, see \eqref{eq:bregman-div-noise}). In Theorem~\ref{thm:shuffling-radius-bound} we bound the shuffling radius in terms of  $\norm{\nabla f(x_\ast)}^2$, $n$,  $L_{\max}$ and the more common quantity $\sigmaesc^2 \eqdef \frac{1}{n} \sum_{i=1}^{n} \sqn{\nabla f_{i} (x_\ast) - \nabla f(x_\ast)}$.

 Secondly, if  $\psi$ is $\mu$-strongly convex, we prove the rate \[
\squeeze
			\ecn{x_T - x_\ast} \leq \br{1 + 2\gamma \mu n}^{-T} \sqn{x_0 - x_\ast} + \frac{ \gamma^2 \sigmarr^2}{\mu},
\]
	where $\gamma_t = \gamma \leq \frac{1}{L_{\max}}$ is the stepsize (see Theorem~\ref{thm:psi-strongly-convex-f-convex})
.

Both mentioned rates show exponential (linear in logarithmic scale) convergence to a neighborhood whose size is proportional to $\gamma^2 \sigmarr^2$. Since we can choose $\gamma$ to be arbitrarily small or periodically decrease it, this implies that the iterates converge to $x_\ast$ in the limit. Moreover, we show in \Cref{sec:strongly_convex} that when $\gamma = \cO(\frac{1}{T})$ the error is $\cO(\frac{1}{T^2})$, which is superior to the $\cO(\frac{1}{T})$ error of SGD.

{\bf Decreasing stepsizes.}
The convergence of RR is not always exact and depends on the parameters of the objective. Similarly, if the shuffling radius $ \sigmarr^2$ is positive, and we wish to find an $\e$-approximate solution, the optimal choice of a fixed stepsize for ProxRR will depend on $\e$. This deficiency can be fixed by using decreasing stepsizes in both vanilla RR~\citep{Ahn2020} and in SGD~\citep{Stich2019b}. We adopt the same technique to our setting. However, we depart from \citep{Ahn2020} by only adjusting the stepsize \emph{once per epoch} rather than at every iteration, similarly to the concurrent work of \cite{tran2020shuffling} on RR with momentum.  For details, see Section~\ref{sec:extensions}.

{\bf Importance sampling for proximal RR.}
While importance sampling is a well established technique for speeding up the convergence of SGD \citep{IProx-SDCA, ES-SGD-nonconvex},  no importance sampling variant of RR has been proposed nor analyzed. This is not surprising since the key property of importance sampling in SGD---unbiasedness---does not hold for RR. Our approach to equip ProxRR with importance sampling is via a reformulation of problem \eqref{eq:finite-sum-min} into a similar problem with a larger number of summands. In particular, for each $i\in [n]$ we include $n_i$ copies of the function $\frac{1}{n_i}f_i$, and then take average of all $N = \sum_i n_i$ functions constructed this way. The value of $n_i$ depends on the ``importance'' of $f_i$, described below.  We then apply ProxRR to this reformulation.

If $f_i$ is $L_i$-smooth for all $i\in [n]$ and we let $\bar{L}\eqdef \frac{1}{n}\sum_i L_i$, then we choose $n_i= \lceil \frac{L_i}{\bar{L}} \rceil$. It is easy to show that $N\leq 2n$, and hence our reformulation leads to at most a doubling of the number of functions forming the finite sum. However, the overall complexity of ProxRR applied to this reformulation will depend on $\bar{L}$ instead of $\max_i L_i$ (see Theorem~\ref{thm:IS}), which can lead to a significant improvement.  For details of the construction and our complexity results, see Section~\ref{sec:extensions}.

{\bf Application to Federated Learning.}
In Section~\ref{sec:FL} we describe an application of our results to federated learning \citep{FEDLEARN, FL2017-AISTATS, FL-big}. 

{\bf Results for SO.}
All of our results apply to the Shuffle-Once algorithm as well. For simplicity, we center the discussion around RR, whose current theoretical guarantees in the non-convex case are better than that of SO. Nevertheless, the other results are the same for both methods, and ProxRR is identical to ProxSO in terms of our theory too. A study of the empirical differences between RR and SO can be found in \citep{MKR2020rr}.

\section{Preliminaries}

In our analysis, we build upon the notions of  {\em limit points} and {\em shuffling variance} introduced by \citet{MKR2020rr} for vanilla (i.e., non-proximal) RR. Given a stepsize $\gamma > 0$ (held constant during each epoch) and a permutation $\pi$ of $\{1,2,\dots,n\}$, the inner loop iterates of RR/SO converge to a neighborhood of intermediate limit points $x_\ast^1, x_\ast^2, \ldots, x_\ast^{n}$ defined by
\begin{equation}
\squeeze
	x_\ast^i \eqdef x_\ast - \gamma \sum \limits_{j=0}^{i-1} \nabla f_{\pi_{j}} (x_\ast), \quad i=1,\dotsc, n-1. \label{eq:x_ast_i}
\end{equation}
The intuition behind this definition is fairly simple: if we performed $i$ steps starting at $x_*$, we would end up close to $x_*^i$. To quantify the closeness, we define the \emph{shuffling radius}.

\begin{Definition}[Shuffling radius]
	\label{def:bregman-div-noise}
	Given a stepsize $\gamma>0$ and a random permutation $\pi$ of $\{ 1, 2, \ldots, n \}$ used in Algorithm~\ref{alg:proxrr}, define $x_\ast^i = x_\ast^i (\gamma, \pi)$ as in \eqref{eq:x_ast_i}. Then, the shuffling radius is defined by
	\begin{equation}\label{eq:bregman-div-noise}  
	\squeeze
	\sigmarr^2 (\gamma) \eqdef \max \limits_{i=1, \ldots, n-1} \left [ \frac{1}{\gamma^2} \mathbb{E}_\pi\bigl[D_{f_{\pi_{i}}} (x_\ast^i, x_\ast)\bigr] \right ], 
	\end{equation}
	where the expectation is taken with respect to the randomness in the permutation $\pi$. If there are multiple stepsizes $\gamma_1, \gamma_2, \ldots$ used in Algorithm~\ref{alg:proxrr}, we take the maximum of all of them as the shuffling radius, i.e., 
	\[ \sigmarr^2 \eqdef \smash{\max_{t=1, 2\dotsc}}\sigmarr^2 (\gamma_t). \]
\end{Definition}

The shuffling radius is related by a multiplicative factor in the stepsize to the shuffling variance introduced by \citet{MKR2020rr}. When the stepsize is held fixed, the difference between the two notions is minimal but when the stepsize is decreasing, the shuffling radius is easier to work with, since it can be upper bounded by problem constants independent of the stepsizes. To prove this upper bound, we rely on a lemma due to \citet{MKR2020rr} that bounds the variance when sampling without replacement.

\begin{Lemma}[Lemma 1 in \citep{MKR2020rr}]
	\label{lemma:sampling-without-replacement}
	Let $X_1, \ldots, X_n \in \R^d$ be fixed vectors, let $\bar{X} = \frac{1}{n} \sum_{i=1}^{n} X_i$ be their mean, and let $\sigma^2 = \frac{1}{n} \sum_{i=1}^{n} \sqn{X_i - \bar{X}}$ be their  variance. Fix any $i \in \{ 1, \ldots, n \}$ and let $X_{\pi_0}, \ldots, X_{\pi_{i-1}}$ be sampled uniformly without replacement from $\{ X_1, \ldots, X_n \}$ and $\bar{X}_{\pi}=\frac{1}{i}\sum_{j=0}^{i-1} X_{\pi_j}$ be their average. Then, the sample average and variance are given by
\begin{equation}\label{eq:b97fg07gdf_08yf8d}
\squeeze
		\ec{\bar{X}_{\pi}} = \bar{X}, \qquad \ecn{\bar{X}_\pi - \bar{X}} = \frac{n-i}{i (n-1)} \sigma^2.
\end{equation}
\end{Lemma}

Armed with \Cref{lemma:sampling-without-replacement}, we can upper bound the shuffling radius using the smoothness constant $L_{\max}$, size of the vector $\nabla f(x_\ast)$ and the variance $\sigmaesc^2$ of the gradient vectors $\nabla f_1(x_\ast)$, $\nabla f_2(x_\ast)$, \dots, $\nabla f_n(x_\ast)$. 

\begin{Theorem}
	\label{thm:shuffling-radius-bound}
	For any stepsize $\gamma > 0$ and any random permutation $\pi$ of $\{1,2,\dots,n\}$ we have
	\[
	\squeeze
		\sigmarr^2 \le \frac{L_{\max}}{2}n\Bigl(n\|\nabla f(x_\ast)\|^2 + \frac{1}{2}\sigmaesc^2\Bigr),
	\]
	where $x_\ast$ is a solution of Problem~\eqref{eq:finite-sum-min} and $\sigmaesc^2$ is the population variance at the optimum 
	\begin{equation}\label{eq:UG(*G(DG(*DGg87gf7ff}
	\squeeze
		\sigmaesc^2 \eqdef \frac{1}{n} \smash{\sum \limits_{i=1}^{n}} \sqn{\nabla f_{i} (x_\ast) - \nabla f(x_\ast)}.
		\end{equation}
\end{Theorem}

All proofs are relegated to the supplementary material. In order to better understand the bound given by \Cref{thm:shuffling-radius-bound}, note that if there is no proximal operator (i.e., $\psi = 0$) then $\nabla f(x_\ast) = 0$ and we get that $\sigmarr^2 \leq \frac{L_{\max} n \sigmaesc^2}{4}$. This recovers the existing upper bound on the shuffling variance of \citet{MKR2020rr} for vanilla RR. On the other hand, if $\nabla f(x_\ast) \neq 0$ then we get an additive term of size proportional to the squared norm of $\nabla f(x_\ast)$. 

\section{Theory for strongly convex losses $f_1, \dots,f_n$}\label{sec:strongly_convex}

Our first theorem establishes a convergence rate for Algorithm~\ref{alg:proxrr} applied with a constant stepsize to Problem~\eqref{eq:finite-sum-min} when each objective $f_i$ is strongly convex. This assumption is commonly satisfied in machine learning applications where each $f_i$ represents a regularized loss on some data points, as in $\ell_2$ regularized linear regression and $\ell_2$ regularized logistic regression.

\begin{Theorem}
	\label{thm:f-strongly-convex-psi-convex}
Let \Cref{as:smooth_fi_proper_psi} be satisfied. 		Further, assume that each $f_i$  is $\mu$-strongly convex. If Algorithm~\ref{alg:proxrr} is run with constant stepsize $\gamma_t = \gamma \leq \frac{1}{L_{\max}}$, then the iterates generated by the algorithm satisfy
\[
\squeeze
			\ecn{x_T - x_\ast} \leq \br{1 - \gamma \mu}^{n T} \sqn{x_0 - x_\ast} + \frac{2 \gamma^2 \sigmarr^2}{\mu}.
\]
\end{Theorem}

We can convert the guarantee of Theorem~\ref{thm:f-strongly-convex-psi-convex} to a convergence rate by properly tuning the stepsize and using the upper bound of \Cref{thm:shuffling-radius-bound} on the shuffling radius. In particular, if we choose the stepsize as
$
\gamma = \min \pbr{ \frac{1}{L_{\max}}, \frac{\sqrt{\e \mu}}{\sqrt{2} \sigmarr} } ,$
then we obtain $\ecn{x_T - x_\ast} = \mathcal{O}\br{\e}$ provided that the total number of iterations $K_{\mathrm{RR}} = n T$ is at least
\begin{equation}
	\label{eq:proxrr-complexity-f-sc}
	\squeeze
        K_{\mathrm{RR}} \geq \biggl (\kappa +  \frac{\sqrt{\kappa n}}{\sqrt{\e} \mu} ( \sqrt{n} \norm{\nabla f(x_\ast)} + \sigmaesc )   \biggr) \log\br{\frac{2 r_0}{\e}},
\end{equation}
where $\kappa\eqdef L_{\max}/\mu$ and $r_0 \eqdef \|x_0 - x_\ast\|^2$.

{\bf Comparison with vanilla RR.} If there is no proximal operator, then $\norm{\nabla f(x_\ast)} = 0$ and we recover the earlier result of \citet{MKR2020rr} on the convergence of RR without proximal, which is optimal in $\e$ up to logarithmic factors. On the other hand, when the proximal operator is nonzero, we get an extra term in the complexity proportional to $\norm{\nabla f(x_\ast)}$: thus, even when all the functions are the same (i.e., $\sigmaesc = 0$), we do not recover the linear convergence of Proximal Gradient Descent \citep{Karimi2016,Beck2017}. This can be easily explained by the fact that \Cref{alg:proxrr} performs $n$ gradient steps per one proximal step. Hence, even if $f_1=\dotsb=f_n$, \Cref{alg:proxrr} does not reduce to Proximal Gradient Descent. We note that other algorithms for composite optimization which may not take a proximal step at every iteration (for example, using stochastic projection steps) also suffer from the same dependence~\citep{patrascu2020stochastic}.

{\bf Comparison with proximal SGD.} 
In order to compare \eqref{eq:proxrr-complexity-f-sc} against the complexity of Proximal SGD (Algorithm~\ref{alg:proxsgd}), we recall the following simple result on the convergence of Proximal SGD. The result is standard \cite{Needell2016,Gower2019}, with the exception that we present it in a slightly generalized in that we also consider the case when $\psi$ is strongly convex. Our proof is a minor modification of that in \cite{Gower2019}, and we offer it in the appendix for completeness.
\begin{algorithm}[t]
    \caption{Proximal SGD}
    \label{alg:proxsgd}
    \begin{algorithmic}[1]
    \Require Stepsizes $\gamma_k > 0$, initial vector $x_0 \in \R^d$, number of steps $K$
        \For{steps $k=0,1,\dotsc,K-1$}
            \State Sample $i_k$ uniformly at random from $[n]$
            \State $x_{k+1} = \prox_{\gamma_k \psi}(x_{k} - \gamma_k \nabla f_{i_k} (x_k))$
        \EndFor
    \end{algorithmic}
\end{algorithm}

\begin{Theorem}[Proximal SGD]
	\label{thm:conv-prox-sgd}
		Let \Cref{as:smooth_fi_proper_psi} hold. 
	Further, suppose that either $f \eqdef \frac{1}{n} \sum_{i=1}^{n} f_i$ is $\mu$-strongly convex or that $\psi$ is $\mu$-strongly convex. If Algorithm~\ref{alg:proxsgd} is run with a constant stepsize $\gamma_k = \gamma > 0$ satisfying $\gamma \leq \frac{1}{2 L_{\max}}$, then the final iterate returned by the algorithm after $K$ steps satisfies
	\[ \squeeze
	\ecn{x_K - x_\ast} \leq \br{1 - \gamma \mu}^{K} \sqn{x_0 - x_\ast} + \frac{2 \gamma \sigmaesc^2}{\mu}.  \] 
\end{Theorem}

Furthermore, by choosing the stepsize $\gamma$ as
$\gamma = \min \pbr{ \frac{1}{2 L_{\max}}, \frac{\e \mu}{4 \sigmaesc} } $, 
we get that $\ecn{x_K - x_\ast} = \mathcal{O}\br{\e}$ provided that the number of iterations is at least
\begin{equation}
	\label{eq:proxsgd-complexity}
	\squeeze
	K_{\mathrm{SGD}} \geq \left( \kappa + \frac{\sigmaesc^2}{\e \mu^2} \right) \log\br{\frac{2 r_0}{\e}}.
\end{equation}

By comparing between the iteration complexities $K_{\mathrm{SGD}}$ (given by \eqref{eq:proxsgd-complexity}) and $K_{\mathrm{RR}}$ (given by~\eqref{eq:proxrr-complexity-f-sc}), we see that ProxRR converges faster than Proximal SGD whenever the target accuracy $\e$ is small enough to satisfy
\[ \squeeze \e \leq \frac{1}{L_{\max} n \mu} \br{\frac{\sigmaesc^4}{n \sqn{\nabla f(x_\ast)} + \sigmaesc^2}}.  \]
Furthermore, the comparison is much better when we consider \emph{proximal iteration complexity} (number of proximal operator access), in which case the complexity of ProxRR~\eqref{eq:proxrr-complexity-f-sc} is reduced by a factor of $n$ (because we take one proximal step every $n$ iterations) while the proximal iteration complexity of Proximal SGD remains the same as~\eqref{eq:proxsgd-complexity}. In this case, ProxRR is better whenever the accuracy $\e$ satisfies
\begin{equation*}
	\begin{split} 
        \e &\squeeze \geq \frac{n}{L_{\max} \mu} \left [ n \sqn{\nabla f(x_\ast)} + \sigmaesc^2 \right ] \\ 
        \qquad \text {or}, 
		\qquad \e &\squeeze \leq \frac{n}{L_{\max} \mu} \left [ \frac{\sigmaesc^4}{n \sqn{\nabla f(x_\ast)} + \sigmaesc^2} \right ].
	\end{split}
\end{equation*}
Therefore we can see that if the target accuracy is large enough or small enough, and if the cost of proximal operators dominates the computation, ProxRR is much quicker to converge than Proximal SGD.

\section{Theory for strongly convex regularizer $\psi$}
In Theorem~\ref{thm:f-strongly-convex-psi-convex}, we assume that each $f_i$ is $\mu$-strongly convex. This is motivated by the common practice of using $\ell_2$ regularization in machine learning. However, applying $\ell_2$ regularization in every step of Algorithm~\ref{alg:proxrr} can be expensive when the data are sparse and the iterates $x_t^i$ are dense, because it requires accessing each coordinate of $x_t^i$ which can be much more expensive than computing sparse gradients $\nabla f_{i} (x_t^i)$. Alternatively, we may instead choose to put the $\ell_2$ regularization inside $\psi$ and only ask that $\psi$ be strongly convex---this way, we can save a lot of time as we need to access each coordinate of the dense iterates $x_t^i$ only once per epoch rather than every iteration. Theorem~\ref{thm:psi-strongly-convex-f-convex} gives a convergence guarantee in this setting.

\begin{Theorem}
	\label{thm:psi-strongly-convex-f-convex}
	Let \Cref{as:smooth_fi_proper_psi} be satisfied. Further, assume that $\psi$ is $\mu$-strongly convex. If Algorithm~\ref{alg:proxrr} is run with constant stepsize $\gamma_t = \gamma \leq \frac{1}{L_{\max}}$, where $L_{\max} = \max_i L_i$, then the iterates generated by the algorithm satisfy
\[\squeeze
			\ecn{x_T - x_\ast} \leq \br{1 + 2\gamma \mu n}^{-T} \sqn{x_0 - x_\ast} + \frac{ \gamma^2 \sigmarr^2}{\mu}.
\]
\end{Theorem}

By making a specific choice for the stepsize used by Algorithm~\ref{alg:proxrr}, we can obtain a convergence guarantee using Theorem~\ref{thm:psi-strongly-convex-f-convex}. Choosing the stepsize as
\begin{equation}
	\label{eq:stepsize-choice}
	\squeeze
	\gamma = \min \pbr{ \frac{1}{L_{\max}}, \frac{\sqrt{\e \mu}}{\sigmarr} }.
\end{equation}

Then $\ecn{x_T - x_\ast} = \mathcal{O}\br{\e}$ provided that the total number of iterations satisfies
\begin{equation}
	\label{eq:convergence-psi-sc-f-c}
	\squeeze
	K \geq \left ( \kappa + \frac{\sigmarr/\mu}{\sqrt{\e \mu}} + n \right ) \log\br{\frac{2 r_0}{\e}}.
\end{equation}
This can be converted to a bound similar to \eqref{eq:proxrr-complexity-f-sc} by using Theorem~\ref{thm:shuffling-radius-bound}, in which case the only difference between the two cases is an extra $n \log\br{\frac{1}{\e}}$ term when only the regularizer $\psi$ is $\mu$-strongly convex. Since for small enough accuracies the $1/\sqrt{\e}$ term dominates, this difference is minimal.

\section{Extensions}\label{sec:extensions}

Before turning to applications, we discuss two extensions to the theory that significantly matter in practice: using decreasing stepsizes and applying importance resampling.  

{\bf Decreasing stepsizes.}
Using the theoretical stepsize \eqref{eq:stepsize-choice} requires knowing the desired accuracy $\e$ ahead of time as well as estimating $\sigmarr$. It also results in extra polylogarithmic factors in the iteration complexity~\eqref{eq:convergence-psi-sc-f-c}, a phenomenon observed and fixed by using decreasing stepsizes in both vanilla RR~\citep{Ahn2020} and in SGD~\citep{Stich2019b}. We show that we can adopt the same technique to our setting. However, we depart from the stepsize scheme of \citet{Ahn2020} by only varying the stepsize \emph{once per epoch} rather than every iteration. This is closer to the common practical heuristic of decreasing the stepsize once every epoch or once every few epochs~\citep{Sun2019, tran2020shuffling}. The stepsize scheme we use is inspired by the schemes of \citep{Stich2019b,ES-SGD-nonconvex}: in particular, we fix $T > 0$, let $t_0 = \ceil{T/2}$, and choose the stepsizes $\gamma_t > 0$ by
\begin{equation}
	\label{def:dec-stepsizes}
	\gamma_{t} = \begin{cases}
		\frac{1}{L_{\max}} & \text { if } T \leq \frac{L_{\max}}{2 \mu n} \text { or } t \leq t_0, \\
		\frac{7}{\mu n \br{s + t - t_0}} & \text { if } T > \frac{L_{\max}}{2 \mu n} \text { and } t > t_0,
	\end{cases}
\end{equation}
where $s \eqdef 7 L_{\max}/(4 \mu n)$. Hence, we fix the stepsize used in the first $T/2$ iterations and then start decreasing it every epoch afterwards. Using this stepsize schedule, we can obtain the following convergence guarantee when each $f_i$ is smooth and convex and the regularizer $\psi$ is $\mu$-strongly convex.

\begin{Theorem}
	\label{thm:psi-strongly-cvx-dec-stepsizes}
	Suppose that each $f_i$ is $L_{\max}$-smooth and convex, and that the regularizer $\psi$ is $\mu$-strongly convex. Fix $T > 0$. Then choosing stepsizes $\gamma_t$ according to \eqref{def:dec-stepsizes} we have that $\gamma_t \leq \frac{1}{L_{\max}}$ for all $t$ and the final iterate generated by Algorithm~\ref{alg:proxrr} satisfies
	\[
		\squeeze
			\ecn{x_{T} - x_\ast} = \mathcal{O}\br{ \exp\br{-\frac{n T}{\kappa + 2n}} r_0 + \frac{\sigmarr^2}{\mu^3 n^2 T^2}},
	\]
	where $\kappa\eqdef L_{\max}/\mu$, $r_0 \eqdef \|x_0 - x_\ast\|^2$ and $\mathcal{O}(\cdot)$ hides absolute (non-problem-specific) constants.
\end{Theorem}

This guarantee holds for any number of epochs $T > 0$. We believe a similar guarantee can be obtained in the case each $f_i$ is strongly-convex and the regularizer $\psi$ is just convex, but we did not include it as it adds little to the overall message.

{\bf Importance resampling.}
Suppose that each $f_i$ is $L_i$-smooth. Then the iteration complexities of both SGD and RR depend on $L_{\max}/\mu$, where $L_{\max}$ is the maximum smoothness constant among the smoothness constants $L_1, L_2, \ldots, L_n$. The maximum smoothness constant can be arbitrarily worse than the average smoothness constant $\Lave = \frac{1}{n} \sum_{i=1}^{n} L_i$. This situation is in contrast to the complexity of gradient descent which depends on the smoothness constant $L_{f}$ of $f = \frac{1}{n} \sum_{i=1}^{n} f_i$, for which we have $L_{f} \leq \Lave$. This is a problem commonly encountered with stochastic optimization methods and may cause significantly degraded performance in practical optimization tasks in comparison with deterministic methods~\citep{tang2019practicality}.

\emph{Importance sampling} is a common technique to improve the convergence of SGD (Algorithm~\ref{alg:proxsgd}): we sample function $\frac{\Lave}{L_i}f_i$ with probability $p_i$ proportional to $L_i$, where $\Lave\eqdef \frac{1}{n}\sum_{i=1}^n L_i$. In that case, the SGD update is still unbiased since 
\[\squeeze \E_i \left[ \frac{\Lave}{L_i}f_i \right] = \sum \limits_{i=1}^n p_i \frac{\Lave}{L_i}f_i=f .\]
Moreover, the smoothness of function $\frac{\Lave}{L_i}f_i$ is $\Lave$ for any $i$, so the guarantees would depend on $\Lave$ instead of $\max_{i=1,\dotsc, n}L_i$. Importance sampling successfully improves the iteration complexity of SGD to depend on $\Lave$~\citep{Needell2016}, and has been investigated in a wide variety of settings~\citep{Gower2018,Gorbunov2020}.

Importance sampling is a neat technique but it relies heavily on the fact that we use \emph{unbiased} sampling. How can we obtain a similar result if inside any permutation the sampling is biased? The answer requires us to think again as to what happens when we replace $f_i$ with $\frac{\Lave}{L_i}f_i$. To make sure the problem remains the same, it is sufficient to have $\frac{\Lave}{L_i}f_i$ inside a permutation exactly $\frac{L_i}{\Lave}$ times. And since $\frac{L_i}{\Lave}$ is not necessarily integer, we should use $n_i=\Bigl\lceil\frac{L_i}{\Lave}\Bigr \rceil$ and solve
\begin{equation}
\squeeze
	\min \limits_{x\in\mathbb{R}^d} \frac{1}{N}\sum \limits_{i=1}^n \Bigl(\underbrace{\tfrac{1}{n_i}f_i(x) +\dotsb + \tfrac{1}{n_i}f_i(x)}_{n_i \text{ times}}  \Bigr) + \psi(x), \label{eq:importance}
\end{equation}
where $N\eqdef n_1+\dotsb +n_n = \Bigl\lceil\frac{L_1}{\Lave}\Bigr \rceil + \dotsb + \Bigl\lceil\frac{L_n}{\Lave}\Bigr \rceil$. Clearly, this problem is equivalent to the original formulation in~\ref{eq:finite-sum-min}. At the same time, we have improved all smoothness constants to $\Lave$. It might seem that that the new problem has more functions, but it turns out that the new number of functions satisfies $N\le 2n$, so any related costs, such as longer loops or storing duplicates of the data, are negligible, as the next theorem shows.

\begin{Theorem}\label{thm:IS}
	For every $i$, assume  that each $f_i$ is convex and $L_i$-smooth, and let $\psi$ be $\mu$-strongly convex. Then, the number of functions $N$ in~\eqref{eq:importance} satisfies $N\le 2n$, and Algorithm~\ref{alg:proxrr} applied to problem~\eqref{eq:importance} has the same complexity as \eqref{eq:convergence-psi-sc-f-c} but proportional to $\Lave$ rather than $L_{\max}$.
\end{Theorem}

\section{Federated learning} \label{sec:FL}

Let us consider now the problem of minimizing the average of $N= \sum_{m=1}^M N_m$ functions that are stored on $M$ devices, which have $N_1,\dotsc, N_M$ samples correspondingly,
\[
\squeeze
	\min \limits_{x\in\R^d} \frac{1}{N}\smash{\sum \limits_{m=1}^M} F_m(x) + R(x), \quad
	F_m(x) = \smash{\sum \limits_{j=1}^{N_m}} f_{mj}(x).\label{eq:fed_learning}
\]
For example, $f_{mj}(x)$ can be the loss associated with a single sample $(X_{mj}, y_{mj})$, where pairs $(X_{mj}, y_{mj})$ follow a distribution $D_m$ that is specific to device $m$. An important instance of such formulation is federated learning, where $M$ devices train a shared model by communicating periodically with a server. We normalize the objective in~\eqref{eq:fed_learning} by $N$ as this is the total number of functions after we expand each $F_m$ into a sum. We denote the solution of~\eqref{eq:fed_learning} by $x_\ast$.

{\bf Extending the space.}
To rewrite the problem as an instance of~\eqref{eq:finite-sum-min}, we are going to consider a bigger product space, which is sometimes used in distributed optimization~\citep{bianchi2015coordinate}. Let us define $n\eqdef \max\{N_1, \dotsc, N_m\}$ and introduce $\psi_C$, the \emph{consensus} constraint,
\[
	\psi_C(x_1,\dotsc, x_M)
	= \begin{cases}
	0, & x_1=\dotsb= x_M\\
	+\infty, & \text{otherwise}
	\end{cases}.
\]
By introducing dummy variables $x_1,\dotsc, x_M$ and adding the constraint $x_1=\dotsb =x_M$, we arrive at the intermediate problem
\[\squeeze
	\min \limits_{x_1,\dotsc, x_M\in \R^p} \frac{1}{N}\sum \limits_{m=1}^M F_m(x_m) + (R + \psi_C)(x_1, \dotsc, x_M),
\]
where $R+\psi_C$ is defined, with a slight abuse of notation, as
\[
	(R+\psi_C)(x_1,\dotsc, x_M)
	= \begin{cases}
	R(x_1), & x_1=\dotsb= x_M\\
	+\infty, & \text{otherwise}.
	\end{cases}
\]
Since we have replaced $R$ with a more complicated regularizer $R+\psi_C$, we need to understand how to compute the proximal operator of the latter. We show (\Cref{lem:ext-proximal-operator} in the supplementary) that the proximal operator of $(R+\psi_C)$ is merely the projection onto $\{(x_1,\dotsc, x_M) \mid x_1=\dotsb = x_M\}$ followed by the proximal operator of $R$ with a smaller stepsize.

{\bf Reformulation.}
To have $n$ functions in every $F_m$, we write $F_m$ as a sum with extra $n-N_m$ zero functions, $f_{mj}(x)\equiv 0$ for any $ j > N_m$, so that
\[
\squeeze
	F_m(x_m) = \smash{\sum \limits_{j=1}^n f_{mj}(x_m) = \sum \limits_{j=1}^{N_m} f_{mj}(x_m)+\sum \limits_{j=N_m+1}^n 0.}
\]
We can now stick the vectors together into $\xx=(x_1,\dotsc, x_M)\in\R^{M\cdot d}$ and multiply the objective by $\frac{N}{n}$, which gives the following reformulation:
\begin{equation}
\squeeze
	\min \limits_{\xx\in\R^{M\cdot d}} \smash{\frac{1}{n} \sum \limits_{i=1}^n} f_i(\xx) + \psi(\xx), \label{eq:fed_reformulation}
\end{equation}
where $ \psi(\xx)\eqdef \frac{N}{n}(R+\psi_C)$ and
\begin{align*}
	&\squeeze f_i(\xx) = f_i(x_1,\dotsc, x_M) \eqdef \smash{\sum \limits_{m=1}^M f_{mi}(x_m)}.
\end{align*}
In other words, function $f_i(\xx)$ includes $i$-th data sample from each device and contains at most one loss from every device, while $F_m(x)$ combines all data losses on device $m$. Note that the solution of~\eqref{eq:fed_reformulation} is $\xx_\ast\eqdef (x_\ast^\top, \dotsc, x_\ast^\top)^\top$ and the gradient of the extended function $f_i(\xx)$ is given by
\[\squeeze
	\nabla f_i (\xx)	= (
		\nabla f_{1i}(x_1)^\top,
		\cdots , 
		\nabla f_{Mi}(x_M)^\top )^\top
\]
Therefore, a stochastic gradient step that uses $\nabla f_i(\xx)$ corresponds to updating all local models with the gradient of $i$-th data sample, without any communication.

\Cref{alg:proxrr} for this specific problem can be written in terms of $x_1,\dotsc, x_M$, which results in \Cref{alg:fed_rr}. Note that since $f_{mi}(x_i)$ depends only on $x_i$, computing its gradient does not require communication. Only once the local epochs are finished, the vectors are averaged as the result of projecting onto the set $\{(x_1,\dotsc, x_M) \mid x_1=\dotsb = x_M\}$. The full description of our FedRR is given in~\Cref{alg:fed_rr}.
\begin{algorithm}[t]
    \caption{Federated Random Reshuffling \rrbox{(FedRR)} and Shuffle-Once \sobox{(FedSO)}}
    \label{alg:fed_rr}
\begin{algorithmic}[1]
   \Require Stepsize $\gamma > 0$, initial vector $x_0 = x_0^0 \in \R^d$, number of epochs $T$
   \State For each $m$, sample permutation $\prm{0, m}, \prm{1, m}, \ldots, \prm{N_m-1, m}$ of $\{ 1, 2, \ldots, N_m \}$ \sobox{(Only FedSO)}
    \For{epochs $t=0,1,\dotsc,T-1$}
    	  \For{$m=1,\dotsc, M$ locally in parallel}
    	  \State $x_{t, m}^0=x_t$
    	  \State Sample permutation $\prm{0, m}, \prm{1, m}, \ldots, \prm{N_m-1, m}$ of $\{ 1, 2, \ldots, N_m \}$ \rrbox{(Only FedRR)}
       \For{$i=0, 1, \ldots, N_m-1$}
          \State $x_{t, m}^{i+1} = x_{t, m}^{i} - \gamma \nabla f_{\prm{i, m}} (x_{t, m}^i)$
       \EndFor
       \State $x_{t, m}^n = x_{t,m}^{N_m}$
       \EndFor
       \State $z_{t+1}=\frac{1}{M}\sum_{m=1}^M x_{t, m}^n$;  $x_{t+1}=\prox_{\gamma \frac{N}{M} R}(z_{t+1})$
    \EndFor
\end{algorithmic}
\end{algorithm}

{\bf Reformulation properties.}
To analyze FedRR, the only thing that we need to do is understand the properties of the reformulation~\eqref{eq:fed_reformulation} and then apply~\Cref{thm:f-strongly-convex-psi-convex} or \Cref{thm:psi-strongly-convex-f-convex}. The following lemma gives us the smoothness and strong convexity properties of \eqref{eq:fed_reformulation}.
\begin{Lemma}\label{lem:fed_reform_properties}
	Let function $f_{mi}$ be $L_i$-smooth and $\mu$-strongly convex for every $m$. Then, $f_i$ from reformulation~\eqref{eq:fed_reformulation} is $L_i$-smooth and $\mu$-strongly convex.
\end{Lemma}

The previous lemma shows that the conditioning of the reformulation is $\kappa=\frac{L_{\max}}{\mu}$ just as we would expect. Moreover, it implies that the requirement on the stepsize remains exactly the same: $\gamma\le \frac{1}{L_{\max}}$. What remains unknown is the value of $\sigmarr^2$, which plays a key role in the convergence bounds for ProxRR and ProxSO. Our next goal, thus, is to obtain an upper bound on $\sigmarr^2$, which would allow us to have a complexity for FedRR and FedSO. To find it, let us define
\[\squeeze
	\sigma_{m, \ast}^2
	\eqdef \smash{\frac{1}{N_m}\sum \limits_{j=1}^{n}} \bigl\|\nabla f_{mj}(x_\ast) - \frac{1}{N_m}\nabla F_m(x_\ast)\bigr\|^2,
\]

which is the variance of local gradients on device $m$. This quantity characterizes the convergence rate of local SGD~\citep{yuan2020federated}, so we should expect it to appear in our bounds too. The next lemma explains how to use it to upper bound $\sigmarr^2$.
\begin{Lemma}\label{lem:fed_sigma}
	The shuffling radius $\sigmarr^2$ of the reformulation~\eqref{eq:fed_reformulation} is upper bounded by
	\[\squeeze
		\sigmarr^2
		\le L_{\max} \smash{\sum \limits_{m=1}^M}\Bigl( \|\nabla F_m(x_\ast)\|^2 + \frac{n}{4}\sigma_{m, \ast}^2\Bigr).
	\]
\end{Lemma}

The lemma shows that the upper bound on $\sigmarr^2$ depends on the sum of local variances $\sum_{m=1}^M \sigma_{m,\ast}^2$ as well as on the local gradient norms $\sum_{m=1}^M\|\nabla F_m(x_\ast)\|^2$. Both of these sums appear in the existing literature on convergence of Local GD/SGD~\citep{Khaled2019a,woodworth2020minibatch, yuan2020federated}.

Equipped with the variance bound, we are ready to present formal convergence results. For simplicity, we will consider heterogeneous and homogeneous cases separately and assume that $N_1=\dotsb=N_M=n$. To further illustrate generality of our results, we will present the heterogeneous assuming strong convexity $R$ and the homogeneous under strong convexity of functions $f_{mi}$.

{\bf Heterogeneous data.}
In the case when the data are heterogeneous, we provide the first local RR method. We can apply either ~\Cref{thm:f-strongly-convex-psi-convex} or \Cref{thm:psi-strongly-convex-f-convex}, but for brevity, we give only the corollary obtained from \Cref{thm:psi-strongly-convex-f-convex}.
\begin{Theorem}\label{thm:fed_hetero}
    Assume that functions $f_{mi}$ are convex and $L_i$-smooth for each $m$ and $i$. If $R$ is $\mu$-strongly convex and $\gamma\le \frac{1}{L_{\max}}$, then we have for the iterates produced by \Cref{alg:fed_rr}
    \begin{align*}
        &\squeeze \ecn{x_T - x_\ast} \leq \br{1 + 2\gamma \mu n}^{-T} \sqn{x_0 - x_\ast}  \\
        &\squeeze \qquad + \frac{ \gamma^2 L_{\max}}{M\mu} \sum \limits_{m=1}^M\Bigl(  \|\nabla F_m(x_\ast)\|^2 + \frac{N}{4M}\sigma_{m, \ast}^2\Bigr).
    \end{align*}
\end{Theorem}

\begin{figure}[t]
\centering
	\includegraphics[scale=0.20]{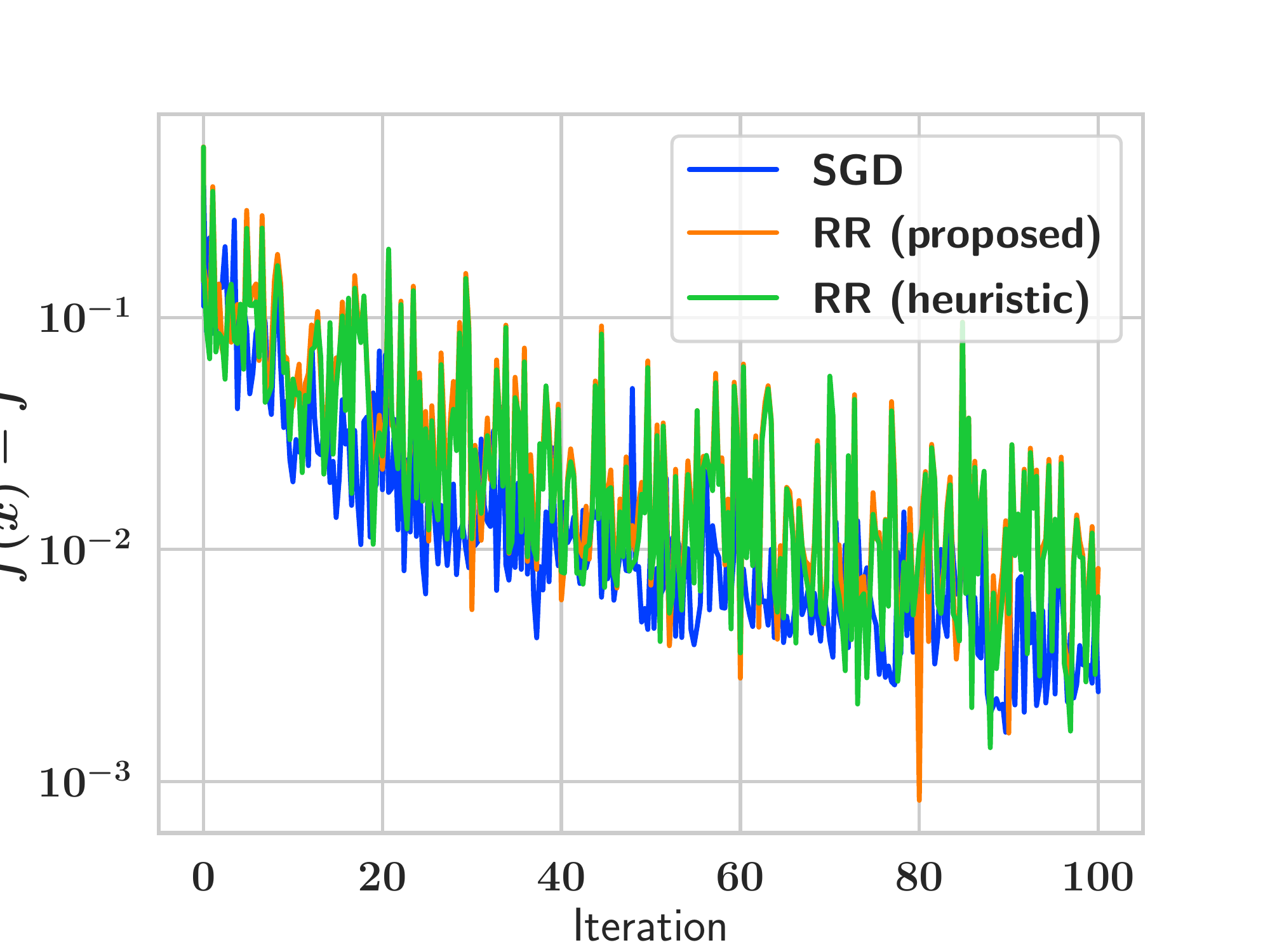}
	\includegraphics[scale=0.20]{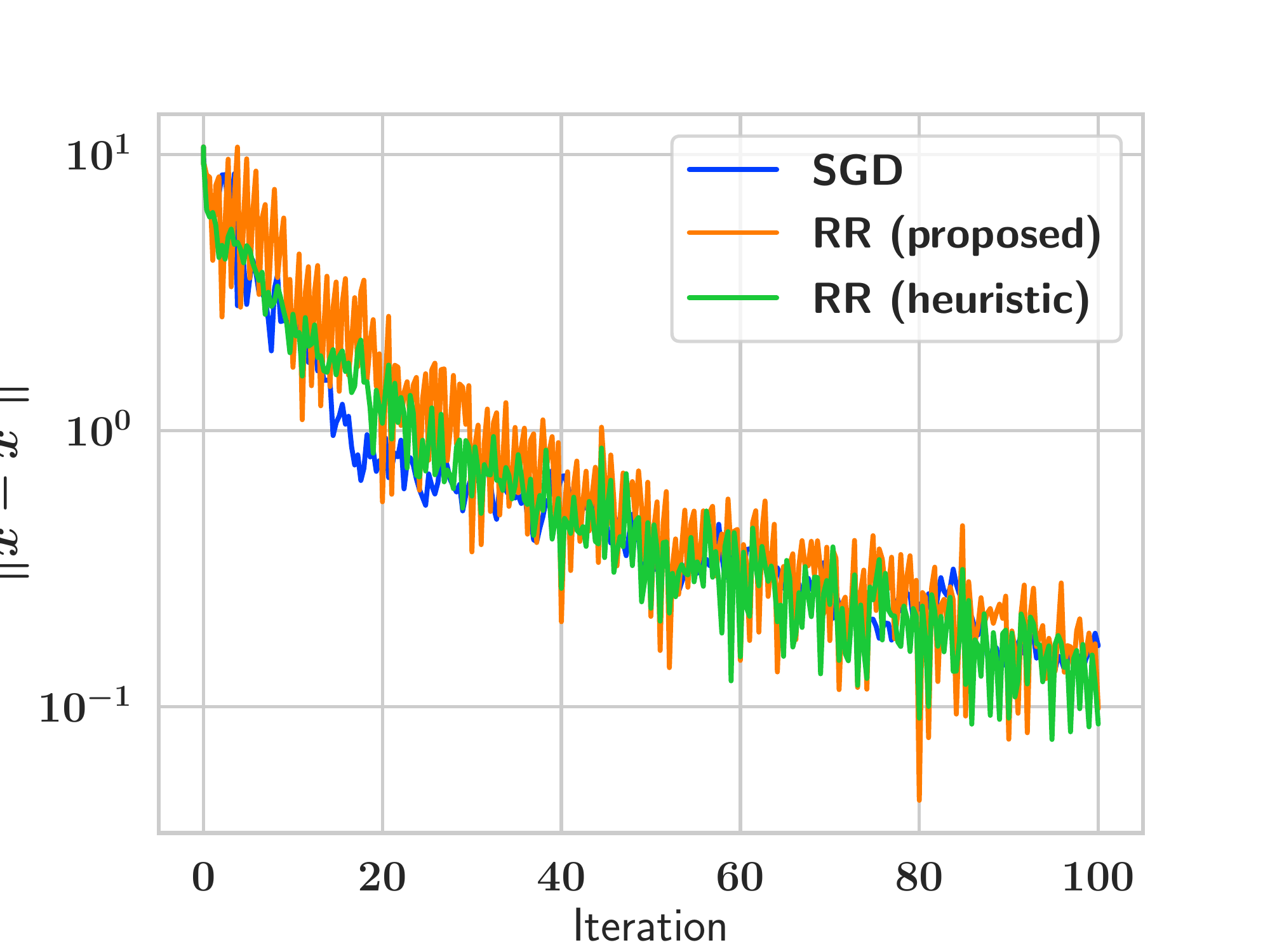}
	\includegraphics[scale=0.20]{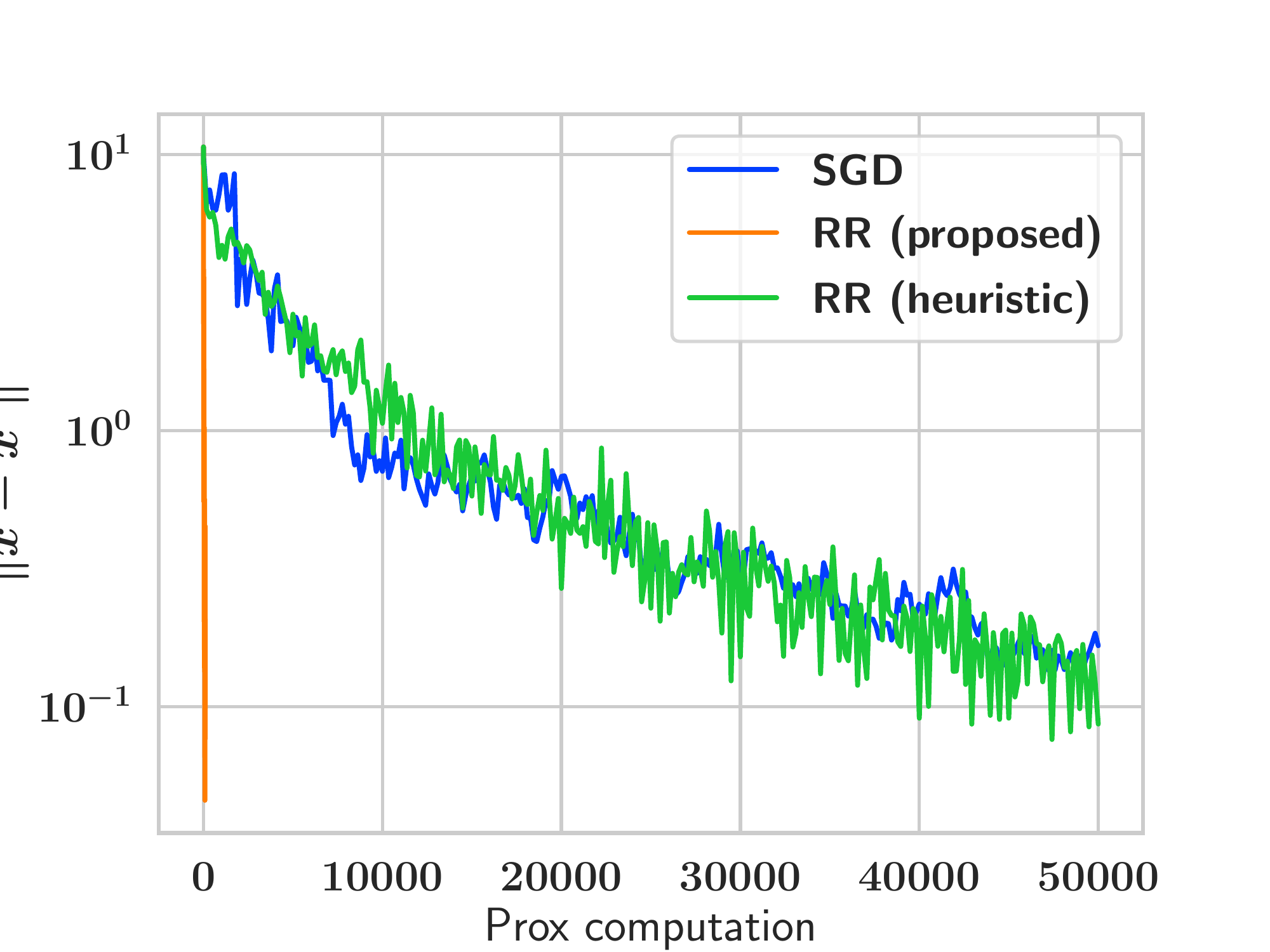}
\caption{Experimental results for problem \eqref{eq:exp_problem}.}
\end{figure}

{\bf Homogeneous data.}
For simplicity, in the homogeneous (i.e., i.i.d.) data case we provide guarantees without the proximal operator. Since then we have $F_1(x)=\dotsb = F_M(x)$, for any $m$ it holds $\nabla F_m(x_\ast)=0$, and thus $\sigma_{m, \ast}^2= \frac{1}{n}\sum_{j=1}^{n} \|\nabla f_{mj}(x_\ast)\|^2$. The full variance is then given by
\[\squeeze
	\sum \limits_{m=1}^M \sigma_{m,\ast}^2 = \frac{1}{n}\sum \limits_{m=1}^M\sum \limits_{i=1}^n \|\nabla f_{mi}(x_\ast)\|^2
	= \frac{N}{n}\sigma_\ast^2
	= M\sigma_\ast^2,
\]
where $\sigma_\ast^2\eqdef \frac{1}{N}\sum_{i=1}^n\sum_{m=1}^M\|\nabla f_{mi}(x_\ast)\|^2$ is the variance of the gradients over all data.
\begin{Theorem}\label{thm:fed_iid}
	Let $R(x)\equiv 0$ (no prox) and the data be i.i.d., that is $\nabla F_m(x_\ast)=0$ for any $m$, where $x_\ast$ is the solution of \eqref{eq:fed_learning}. If each $f_{mj}$ is $L_{\max}$-smooth and $\mu$-strongly convex, then the iterates of \Cref{alg:fed_rr} satisfy
	\[\squeeze
		\ec{\|x_T - x_\ast\|^2}
		\le (1-\gamma\mu)^{nT}\|x_0-x_\ast\|^2 + \frac{\gamma^2L_{\max}N\sigma_\ast^2}{M\mu},
	\]
	where $\sigma_\ast^2\eqdef \frac{1}{N}\sum_{i=1}^n\sum_{m=1}^M\|\nabla f_{mi}(x_\ast)\|^2$.
\end{Theorem}

The most important part of this result is that the last term in \Cref{thm:fed_iid} has a factor of $M$ in the denominator, meaning that the convergence bound improves with the number of devices involved.

\section[Experiments]{Experiments\footnote{Our code: \href{https://github.com/konstmish/rr_prox_fed}{https://github.com/konstmish/rr\_prox\_fed}}}

We look at the logistic regression loss with the elastic net regularization,
\begin{equation}\squeeze \label{eq:exp_problem}
	\smash{\frac{1}{N}\sum \limits_{i=1}^N f_i (x) + \lambda_1\|x\|_1+ \frac{\lambda_2}{2}\|x\|^2},
\end{equation}
where each $f_i: \R^d \to \R$ is defined as
\[ f_i (x) \eqdef -\big(b_i \log \big(h(a_i^\top x)\big) + (1-b_i)\log\big(1-h(a_i^\top x)\big)\big) \]
and where $(a_i, b_i)\in \R^d\times \{0, 1\}$, $i=1,\dotsc, N$ are the data samples, $h\colon t\to1/(1+e^{-t})$ is the sigmoid function, and $\lambda_1, \lambda_2\ge 0$ are parameters.  We set minibatch sizes to 1 for all methods and use theoretical stepsizes, without any tuning. We denote the version of RR that performs proximal operator step after each iteration as `RR (heuristic)'. We give more details in the supplementary. From the experiments, we can see that all methods behave more or less the same way. However, the algorithm that we propose needs only a small fraction of proximal operator evaluations, which gives it a huge advantage whenever the operator takes more time to compute than stochastic gradients.

\bibliography{rr_prox_icml}
\bibliographystyle{icml2021}

\clearpage
\onecolumn
\part*{Supplementary Material}



\section{Basic notions and preliminaries}
\label{sec:basic_notions}
We say that an extended real-valued function $\phi:\R^d\to \R\cup \{+\infty\}$ is proper if its domain, ${\rm dom} \; \phi \eqdef \{x : \phi(x)<+\infty\}$, is nonempty.  We say that it is convex (resp.\ closed) if its epigraph, ${\rm epi}\; \phi \eqdef \{(x,t) \in \R^d\times \R \;:\; \phi(x) \leq t\}$, is a convex (resp.\ closed) set. Equivalently, $\phi$ is convex if ${\rm dom} \; \phi $ is a convex set and $\phi(\alpha x + (1-\alpha)y) \leq \alpha \phi(x) + (1-\alpha) \phi(y)$ for all $x,y\in {\rm dom} \; \phi$ and $\alpha\in(0, 1)$. Finally, $\phi$ is $\mu$-strongly convex if $\phi (x) -\frac{\mu}{2}\norm{x}^2$ is convex, and $L$-smooth if $\frac{L}{2}\norm{x}^2-\phi(x)$ is convex.

These notions have a more useful characterization in the case of real valued and continuously differentiable functions $\phi:\R^d\to \R$.  The Bregman divergence of such $\phi$ is defined by
$ D_{\phi} (x, y) \eqdef \phi(x) - \phi(y) - \ev{\nabla \phi(y), x - y}. $
A continuously differentiable function $\phi$ is called $\mu$-strongly convex if 
\[\tfrac{\mu}{2}\norm{x - y}^2  \leq D_\phi(x,y), \qquad \forall x,y\in \R^d. \]
 It is convex if this holds with $\mu=0$. Moreover, a continuously differentiable function $\phi$ is called $L$-smooth if
\begin{equation} \label{eq:L-smooth-intro} D_\phi(x,y) \leq \tfrac{L}{2} \norm{x - y}^2, \qquad \forall x,y\in \R^d. \end{equation}
Finally, we define $[n]\eqdef \{1,2,\dots,n\}$.

\subsection{Properties of the proximal operator}

Before we proceed to the proofs of convergence, we should state some basic and well-known properties of the regularized objectives. The following lemma explains why the solution of~\eqref{eq:finite-sum-min} is a fixed point of the proximal-gradient step for \emph{any} stepsize.
\begin{Lemma}
	\label{prop:fixed-point} Let \Cref{as:smooth_fi_proper_psi} be satisfied.\footnote{We only need the part about $\psi$.} 	Then point $x_\ast$ is a minimizer of $P(x) =f(x) + \psi (x)$ if and only if for any $\gamma, b > 0$ we have
	\[ x_\ast = \prox_{\gamma n \psi} (x_\ast - \gamma b \nabla f(x_\ast)). \]
\end{Lemma}
\begin{proof}
	This follows by writing the first-order optimality conditions for problem~\eqref{eq:finite-sum-min}, see \citep[p.32]{Neal2014} for a full proof.
\end{proof}

The lemma above only shows that proximal-gradient step does not hurt if we are at the solution. In addition, we will rely on the following a bit stronger result which postulates that the proximal operator is a contraction (resp.\ strong contraction) if the regularizer $\psi$ is convex (resp.\ strongly convex). 
\begin{Lemma}
	\label{prop:prox-contraction} Let \Cref{as:smooth_fi_proper_psi} be satisfied.\footnote{We only need the part about $\psi$.}  	If $\psi$ is $\mu$-strongly convex with $\mu\ge 0$, then for any $\gamma>0$ we have
	\begin{equation}
		\|\prox_{\gamma n\psi}(x)-\prox_{\gamma n\psi}(y)\|^2
		\le \frac{1}{1+2\gamma\mu n}\|x - y\|^2, \label{eq:prox_non_exp}
    \end{equation}
    for all $x, y \in \R^d$.
\end{Lemma}
\begin{proof}
	Let $u\eqdef \prox_{\gamma n\psi}(x)$ and $v\eqdef \prox_{\gamma n\psi}(y)$. By definition, $u=\argmin_w \{\psi(w) + \frac{1}{2\gamma n}\|w-x\|^2 \}$. By first-order optimality, we have $0\in \partial \psi(u)+ \frac{1}{\gamma n}(u - x)$ or simply $x-u\in  \gamma n\partial\psi(u)$. Using a similar argument for $v$, we get $x-u-(y-v)\in \gamma n(\partial \psi(u)-\partial\psi(v))$. Thus, by strong convexity of $\psi$, we get
	\[
		\<x-u-(y-v), u-v>\ge \gamma \mu n\|u-v\|^2.
	\]
	Hence,
	\begin{align*}
		\|x-y\|^2
		&= \|u - v + (x-u - (y-v)) \|^2  \\
		&=\|u - v\|^2 + 2\<x-u - (y-v), u-v> + \|x-u-(y-v)\|^2 \\
		&\ge \|u - v\|^2 + 2\<x-u - (y-v), u-v>  \\
		&\ge (1+2\gamma \mu n)\|u-v\|^2. \newqedhere
	\end{align*}
\end{proof}

\section{Proof of \Cref{thm:shuffling-radius-bound}}
\begin{proof}
	By the $L_{i}$-smoothness of $f_i$ and the definition of $x_\ast^i$, we can replace the Bregman divergence in \eqref{eq:bregman-div-noise}  with the bound
		\begin{align}
		\ec{D_{f_{\pi_{i}}} (x_\ast^i, x_\ast)} \; \overset{\eqref{eq:L-smooth-intro}}{\le} \; \ec{\frac{L_{\pi_i}}{2}\norm{x_\ast^i - x_\ast}^2}  &\le   \frac{L_{\max}}{2}\ec{\|x_\ast^i - x_\ast\|^2} \notag \\ 
		& \overset{\eqref{eq:x_ast_i}}{=}   \frac{\gamma^2 L_{\max}}{2} \mathbb{E}\Biggl[\biggl\| \sum_{j=0}^{i-1}\nabla f_{\pi_j}(x_\ast) \biggr\|^2\Biggr]\notag  \\
		&=  \frac{\gamma^2 L_{\max} i^2}{2} \mathbb{E}\Biggl[\biggl\|\frac{1}{i} \sum_{j=0}^{i-1}\nabla f_{\pi_j}(x_\ast) \biggr\|^2\Biggr] \notag \\
		&= \frac{\gamma^2 L_{\max} i^2}{2}  \ecn{\bar{X}_\pi },\label{eq:8g9fd8gf9d}
	\end{align}
	where $\bar{X}_\pi = \frac{1}{j}\sum_{j=0}^{i-1} X_{\pi_j}$ with  $X_j \eqdef  \nabla f_{j}(x_\ast)$ for $j=1,2,\dots,n$.  Since $\bar{X} = \nabla f(x_\ast)$,  by applying Lemma~\ref{lemma:sampling-without-replacement} we get
    \begin{equation} 
        \ecn{\bar{X}_\pi } \;= \;   \norm{\bar{X}}^2 + \ecn{\bar{X}_\pi - \bar{X}}
        \; \overset{\eqref{eq:b97fg07gdf_08yf8d} + \eqref{eq:UG(*G(DG(*DGg87gf7ff}}{=} \; \norm{\nabla f(x_\ast)}^2 +\frac{n-i}{i (n-1)} \sigmaesc^2. \label{eq:08gfd_898fd*(*gdJS}
    \end{equation}
    It remains to combine  \eqref{eq:8g9fd8gf9d}  and \eqref{eq:08gfd_898fd*(*gdJS}, use the bounds  $i^2\leq n^2$ and
    $i(n-i)\le \frac{n(n-1)}{2}$, which holds for all $i\in \{1,2,  \dots, n-1\}$, and divide both sides of the resulting inequality by $\gamma^2$.
\end{proof}

\section{Main convergence proofs}
\subsection{A key lemma for shuffling-based methods}
The intermediate limit points $x_\ast^i$ are extremely important for showing tight convergence guarantees for Random Reshuffling even without proximal operator. The following lemma illustrates that by giving a simple recursion, whose derivation follows \citep[Proof of Theorem 1]{MKR2020rr}. The proof is included for completeness.

\begin{Lemma}[Theorem 1 in \citep{MKR2020rr}]
	\label{lemma:inner-loop}
	Suppose that each $f_i$ is $L_i$-smooth and $\lambda$-strongly convex (where $\lambda = 0$ means each $f_i$ is just convex). Then the inner iterates generated by Algorithm~\ref{alg:proxrr} satisfy
	\begin{align}
		\ecn{x_t^{i+1} - x_\ast^{i+1}} \leq \br{1 - \gamma \lambda} \ecn{x_t^i - x_\ast^i} - 2 \gamma \br{1 - \gamma L_{\max}} \ec{D_{f_{\pi_i}} (x_t^i, x_\ast)} + 2 \gamma^3 \sigmarr^2,		\label{eq:inner-loop-recursion}
	\end{align}
	where $x_\ast^i $ is as in \eqref{eq:x_ast_i} and $x_\ast$ is any minimizer of $P$.
\end{Lemma}
\begin{proof}
	By definition of $x_{t}^{i+1}$ and $x_\ast^{i+1}$, we have
	\begin{align}
		\begin{split}
			\ecn{x_t^{i+1}-x_*^{i+1}}  &= \ecn{x_t^{i}-x_*^{i}}-2 \gamma \ec{\<\nabla f_{\pi_i}(x_t^i)-\nabla f_{\pi_i}(x_*), x_t^i - x_*^i> } \\
			& \qquad + \gamma^2 \ecn{\nabla f_{\pi_i}(x_t^i) - \nabla f_{\pi_i}(x_*)}. 		
		\end{split}\label{eq:inner-loop-proof-1}
	\end{align}
	Note that the third term in \eqref{eq:inner-loop-proof-1} can be bounded as
	\begin{equation}
		\label{eq:inner-loop-proof-2}
		\sqn{\nabla f_{\pi_i} (x_t^i) - \nabla f_{\pi_i} (x_t^i)} \leq 2 L_{\max} \cdot D_{f_{\pi_i}} (x_t^i, x_\ast).
	\end{equation}
	We may rewrite the second term in \eqref{eq:inner-loop-proof-1} using the three-point identity \citep[Lemma 3.1]{Chen1993} as
	\begin{equation}
		\label{eq:inner-loop-proof-3}
		\ev{\nabla f_{\pi_i} (x_t^i) - \nabla f_{\pi_{i}} (x_\ast), x_t^i - x_\ast^i } = D_{f_{\pi_i}} (x_\ast^i, x_t^i) + D_{f_{\pi_i}} (x_t^i, x_\ast) - D_{f_{\pi_i}} (x_\ast^i, x_\ast).
	\end{equation}
	Combining \eqref{eq:inner-loop-proof-1}, \eqref{eq:inner-loop-proof-2}, and \eqref{eq:inner-loop-proof-3} we obtain
	\begin{align}
		\begin{split}
			\ecn{x_t^{i+1} - x_\ast^{i+1}} \leq \ecn{x_t^i - x_\ast^i} &- 2 \gamma \cdot \ec{D_{f_{\pi_i}} (x_\ast^i, x_t^i)} + 2 \gamma \cdot \ec{D_{f_{\pi_i}} (x_\ast^i, x_\ast)} \\ 
			&- 2 \gamma \br{1 - \gamma L_{\max}} \ec{D_{f_{\pi_i}} (x_t^i, x_\ast)}.	
		\end{split}\label{eq:inner-loop-proof-4}
	\end{align}
	Using  $\lambda$-strong convexity of $f_{\pi_i}$, we derive
	\begin{equation}
		\label{eq:inner-loop-proof-5}
		\frac{\lambda}{2} \sqn{x_t^i - x_\ast^i} \leq D_{f_{\pi_i}} (x_\ast^i, x_t^i).
	\end{equation}
	Furthermore, by the definition of shuffling radius (Definition~\ref{def:bregman-div-noise}), we have
	\begin{equation}
		\label{eq:inner-loop-proof-6}
		\ec{D_{f_{\pi_i}} (x_\ast^i, x_\ast)} \leq \max_{i=1, \ldots, n-1}  \ec{D_{f_{\pi_i}} (x_\ast^i, x_\ast) }  = \gamma^2 \sigmarr^2.
	\end{equation}
	Using~\eqref{eq:inner-loop-proof-5} and~\eqref{eq:inner-loop-proof-6} in \eqref{eq:inner-loop-proof-4} yields \eqref{eq:inner-loop-recursion}.
\end{proof}

\subsection{Proof of \Cref{thm:f-strongly-convex-psi-convex}}
\begin{proof}
	Starting with Lemma~\ref{lemma:inner-loop} with $\lambda = \mu$, we have
	\begin{align*}
		\begin{split}
			\ecn{x_t^{i+1} - x_\ast^{i+1}} \leq \br{1 - \gamma \mu} \ecn{x_t^i - x_\ast^i} - 2 \gamma \br{ 1 - \gamma L_{\max}} \ec{D_{f_{\pi_i}} (x_t^i, x_\ast) } + 2 \gamma^3 \sigmarr^2.
		\end{split}
	\end{align*}
	Since $D_{f_{\pi}} (x_t^i, x_\ast)$ is a Bregman divergence of a convex function, it is nonnegative. Combining this with the fact that the stepsize satisfies $\gamma \leq \frac{1}{L_{\max}}$, we have
	\[ \ecn{x_t^{i+1} - x_\ast^{i+1}} \leq \br{1 - \gamma \mu} \ecn{x_t^i - x_\ast^i} + 2 \gamma^3 \sigmarr^2. \]
	Unrolling this recursion for $n$ steps, we get
	\begin{align}
		\ecn{x_t^{n} - x_\ast^n} &\leq \br{1 - \gamma \mu}^{n} \ecn{x_t^0 - x_\ast^0} + 2 \gamma^3 \sigmarr^2 \br{ \sum_{j=0}^{n-1} \br{1 - \gamma \mu}^{j} } \notag \\
		&= \br{1 - \gamma \mu}^{n} \ecn{x_t - x_\ast} + 2 \gamma^3 \sigmarr^2 \br{ \sum_{j=0}^{n-1} \br{1 - \gamma \mu}^j },		\label{eq:thm-f-sc-proof-1}
	\end{align}
	where we used the fact that $x_t^0 - x_\ast^0 = x_t - x_\ast$. Since $x_\ast$ minimizes $P$, we have by Lemma~\ref{prop:fixed-point} that
	\[ x_\ast = \prox_{\gamma n \psi} \br{ x_\ast - \gamma \sum_{i=0}^{n-1} \nabla f_{\pi_i} (x_\ast) } = \prox_{\gamma n \psi} \br{x_\ast^n}. \] 
	Moreover, by Lemma~\ref{prop:prox-contraction} we obtain that 
	\[ \sqn{x_{t+1} - x_\ast} = \sqn{\prox_{\gamma n \psi} (x_t^n) - \prox_{\gamma n \psi} (x_\ast^n) } \leq \sqn{x_t^n - x_\ast^n}. \]
	Using this in~\eqref{eq:thm-f-sc-proof-1} yields
	\[
		\ecn{x_{t+1} - x_\ast} \leq \br{1 - \gamma \mu}^n \ecn{x_t - x_\ast} + 2 \gamma^3 \sigmarr^2 \br{ \sum_{j=0}^{n-1} \br{1 - \gamma \mu}^j }.
	\]
	We now unroll this recursion again for $T$ steps
	\begin{align}
		\label{eq:thm-f-sc-proof-2}
		\begin{split}
			\ecn{x_{T} - x_\ast} \leq \br{1 - \gamma \mu}^{n T} \ecn{x_0 - x_\ast} + 2 \gamma^3 \sigmarr^2 \br{ \sum_{j=0}^{n-1} \br{1 - \gamma \mu}^j } \br{ \sum_{i=0}^{T-1} \br{1 - \gamma \mu}^{n i} }
		\end{split}.
	\end{align}
	Following \citet{MKR2020rr}, we rewrite and bound the product in the last term as
	\begin{align*}
		\br{ \sum_{j=0}^{n-1} \br{1 - \gamma \mu}^j } \br{ \sum_{i=0}^{T-1} \br{1 - \gamma \mu}^{n i} } &= \sum_{j=0}^{n-1} \sum_{i=0}^{T-1} \br{1 - \gamma \mu}^{ni + j} \\
		&= \sum_{k=0}^{nT-1} \br{1 - \gamma \mu}^k \\
		&\leq \sum_{k=0}^{\infty} \br{1 - \gamma \mu}^{k} \quad = \frac{1}{\gamma \mu}.
	\end{align*}
	It remains to plug this bound into~\eqref{eq:thm-f-sc-proof-2}.
\end{proof}

\subsection{Proof of \Cref{thm:psi-strongly-convex-f-convex}}
\begin{proof}
	Starting with Lemma~\ref{lemma:inner-loop} with $\lambda = 0$, we have
	\[ \ecn{x_t^{i+1} - x_\ast^{i+1}} \leq \ecn{x_t^i - x_\ast^i} - 2 \gamma \br{1 - \gamma L_{\max} } \ec{D_{f_{\pi_i}} (x_t^i, x_\ast) } + 2 \gamma^3 \sigmarr^2. \]
	Since $\gamma \leq \frac{1}{L_{\max}}$ and $D_{f_{\pi}} (x_t^i, x_\ast)$ is nonnegative we may simplify this to
	\[ \ecn{x_t^{i+1} - x_\ast^{i+1}} \leq \ecn{x_t^i - x_\ast^i} + 2 \gamma^3 \sigmarr^2. \]
	Unrolling this recursion over an epoch we have
	\begin{equation}
		\label{eq:thm-psi-sc-proof-1}
		\ecn{x_t^{n} - x_\ast^n} \leq \ecn{x_t^0 - x_\ast^0} + 2 \gamma^3 \sigmarr^2 n = \ecn{x_t - x_\ast} + 2 \gamma^3 \sigmarr^2 n.
	\end{equation}
	Since $x_\ast$ minimizes $P$, we have by Lemma~\ref{prop:fixed-point} that
	\[ x_\ast = \prox_{\gamma n \psi} \br{ x_\ast - \gamma \sum_{i=0}^{n-1} \nabla f_{\pi_i} (x_\ast) } = \prox_{\gamma n \psi} \br{x_\ast^n}. \] 
	Hence, $x_{t+1} - x_\ast = \prox_{\gamma n \psi} (x_t^n) - \prox_{\gamma n \psi} (x_\ast^n)$. We may now use Lemma~\ref{prop:prox-contraction} to get
	\[ \br{1 + 2 \gamma \mu n} \ecn{x_{t+1} - x_\ast} \leq \ecn{x_t^n - x_\ast^n}. \]
	Combining this with \eqref{eq:thm-psi-sc-proof-1}, we obtain
	\[
		\ecn{x_{t+1} - x_\ast} \leq \frac{1}{1 +2 \gamma \mu n} \ecn{x_t - x_\ast} + \frac{2 \gamma^3 \sigmarr^2 n}{1 + 2\gamma \mu n}.
	\]
	We may unroll this recursion again, this time for $T$ steps, and then use that $\sum_{j=1}^{T-1} \br{1 +2 \gamma \mu n}^{-j} \leq \sum_{j=1}^{\infty} \br{1 + 2\gamma \mu n}^{-j} = 1/(2\gamma \mu n)$:
	\begin{align*}
		\ecn{x_{T} - x_\ast} &\leq \br{1 + 2\gamma \mu n}^{-T} \ecn{x_0 - x_\ast} + \frac{2 \gamma^3 \sigmarr^2 n}{1 +2 \gamma \mu n} \Biggl(\sum_{j=0}^{T-1} \br{1 +2 \gamma \mu n}^{-j}\Biggr) \\
		&= \br{1 +2 \gamma \mu n}^{-T} \ecn{x_0 - x_\ast} + 2 \gamma^3 \sigmarr^2 n \Biggl(\sum_{j=1}^{T} \br{1 +2 \gamma \mu n}^{-j} \Biggr) \\
		&\leq \br{1 + 2\gamma \mu n}^{-T} \ecn{x_0 - x_\ast} + 2 \gamma^3 \sigmarr^2 n \frac{1}{2\gamma \mu n} \\
		&= \br{1 +2 \gamma \mu n}^{-T} \ecn{x_0 - x_\ast} + \frac{ \gamma^2 \sigmarr^2}{\mu}.
		\newqedhere
	\end{align*}
\end{proof}

\section{Convergence of SGD (Proof of \Cref{thm:conv-prox-sgd})}
\begin{proof}
	We will prove the case when $\psi$ is $\mu$-strongly convex. The other result follows as a straightforward special case of \citep[Theorem 4.1]{Gorbunov2020}. We start by analyzing one step of SGD with stepsize $\gamma_k = \gamma$ and using \Cref{prop:fixed-point}
	\begin{align}
		\sqn{x_{k+1} - x_\ast} &= \sqn{ \prox_{\gamma \psi} (x_k - \gamma \nabla f_{\xi} (x_k)) - \prox_{\gamma \psi} (x_\ast - \gamma \nabla f (x_\ast))} \nonumber \\
		&\leq \frac{1}{1 + 2 \gamma \mu} \sqn{x_k - \gamma \nabla f_{\xi} (x_k) - (x_\ast - \gamma \nabla f (x_\ast))}. \label{eq:sgd-proof-1}
	\end{align}
	We may write the squared norm term in \eqref{eq:sgd-proof-1} as
	\begin{align}
		\label{eq:sgd-proof-2}
		\begin{split}
			\sqn{x_k - \gamma \nabla f_{\xi} (x_k) - (x_\ast - \gamma \nabla f (x_\ast))} &= \sqn{x_k - x_\ast} - 2 \gamma \ev{x_k - x_\ast, \nabla f_{\xi} (x_k) - \nabla f (x_\ast)} \\
			&\qquad + \gamma^2 \sqn{\nabla f_{\xi} (x_k) - \nabla f (x_\ast)}.
		\end{split}
	\end{align}
	We denote by $\ec[k]{\cdot}$ expectation conditional on $x_k$. Note that the gradient estimate is conditionally unbiased, i.e., that $\ec[k]{\nabla f_\xi (x_k)} = \frac{1}{n} \sum_{i=1}^{n} \nabla f_i (x_k) = \nabla f(x_k)$. Hence, taking conditional expectation in \eqref{eq:sgd-proof-2} and using unbiasedness we have
	\begin{align}
		\label{eq:sgd-proof-3}
		\begin{split}
			\ecn[k]{x_k - \gamma \nabla f_{\xi} (x_k) - (x_\ast - \gamma \nabla f (x_\ast))} &= \sqn{x_k - x_\ast} - 2 \gamma \ev{x_k - x_\ast, \nabla f (x_k) - \nabla f (x_\ast)} \\
			&\qquad + \gamma^2 \ecn[k]{\nabla f_{\xi} (x_k) - \nabla f (x_\ast)}.
		\end{split}
	\end{align}
	By the convexity of $f$ we have
	\[
		\ev{x_k - x_\ast, \nabla f(x_k) - \nabla f(x_\ast)} \geq D_{f} (x_k, x_\ast).
	\]
	Furthermore, we may estimate the third term in \eqref{eq:sgd-proof-3} by first using the fact that $\sqn{x + y} \leq 2 \sqn{x} + 2 \sqn {y}$ for any two vectors $x, y \in \R^d$
	\begin{align*}
		\ec[k]{\sqn{\nabla f_{\xi} (x_k) - \nabla f(x_\ast)}} &\leq 2 \ecn[k]{\nabla f_{\xi} (x_k) - \nabla f_{\xi} (x_\ast)} + 2 \ecn[k]{\nabla f_{\xi} (x_\ast) - \nabla f(x_\ast)} \nonumber \\
		&= 2 \ecn[k]{\nabla f_{\xi} (x_k) - \nabla f_{\xi} (x_\ast)} + 2 \sigmaesc^2.
	\end{align*}
	We now use that by the $L_{\max}$-smoothness of $f_i$ we have that 
	\[ \sqn{\nabla f_i (x_k) - \nabla f_{i} (x_\ast)} \leq 2 L_{\max} \cdot D_{f_i} (x_k, x_\ast). \] 
	Hence
	\begin{align}
		\ecn[k]{\nabla f_{\xi} (x_k) - \nabla f_{\xi} (x_\ast)} &= \frac{1}{n} \sum_{i=1}^{n} \sqn{\nabla f_i (x_k) - \nabla f_{i} (x_\ast)} \nonumber \\
		&\leq \frac{2 L_{\max}}{n} \sum_{i=1}^{n} \left [ f_i (x_k) - f_i (x_\ast) - \ev{\nabla f_i (x_\ast), x_k - x_\ast} \right ] \nonumber \\
		&= 2 L_{\max} \left [ f(x_k) - f(x_\ast) - \ev{\nabla f(x_\ast), x_k - x_\ast} \right ] \nonumber \\
		&= 2 L_{\max} D_{f} (x_k, x_\ast).
		\label{eq:sgd-proof-6}
	\end{align}
	Combining equations \eqref{eq:sgd-proof-3}--\eqref{eq:sgd-proof-6} we obtain
	\begin{align*}
		\begin{split}
			\ecn[k]{x_k - \gamma \nabla f_\xi (x_k) - (x_\ast - \gamma \nabla f(x_\ast))} &\leq \sqn{x_k - x_\ast} - 2 \gamma \br{1 - 2 \gamma L_{\max}} D_{f} (x_k, x_\ast) \\
			&\qquad + 2 \gamma^2 \sigmaesc^2.
		\end{split}
	\end{align*}
	Since $\gamma \leq \frac{1}{2 L_{\max}}$ by assumption we have that $1 - 2 \gamma L_{\max} \geq 0$. Since $D_{f} (x_k, x_\ast) \geq 0$ by the convexity of $f$ we arrive at
	\[
		\ecn[k]{x_k - \gamma \nabla f_\xi (x_k) - (x_\ast - \gamma \nabla f(x_\ast))} \leq \sqn{x_k - x_\ast} + 2 \gamma^2 \sigmaesc^2.
	\]
	Taking unconditional expectation and combining \eqref{eq:dec-step-1} with the last equation we have
	\begin{align*}
		\ecn{x_{k+1} - x_\ast} &\leq \frac{1}{1 + 2 \gamma \mu} \br{ \ecn{x_k - x_\ast} + 2 \gamma^2 \sigmaesc^2 } \\
		&= \frac{1}{1 + 2 \gamma \mu} \ecn{x_k - x_\ast} + \frac{2 \gamma^2 \sigmaesc^2}{1 + 2 \gamma \mu} \\
		&\leq \frac{1}{1 + 2 \gamma \mu} \ecn{x_k - x_\ast} + 2 \gamma^2 \sigmaesc^2.
	\end{align*}
	To simplify this further, we use that for any $x \leq \frac{1}{2}$ we have that $\frac{1}{1 + 2x} \leq 1-x$ and that $\gamma \mu \leq \frac{\mu}{2 L_{\max}} \leq \frac{1}{2}$, hence
	\begin{align*}
		\ecn{x_{k+1} - x_\ast} \leq \br{1 - \gamma \mu} \ecn{x_k - x_\ast} + 2 \gamma^2 \sigmaesc^2.
	\end{align*}
	Recursing the above inequality for $K$ steps yields
	\begin{align*}
		\ecn{x_K - x_\ast} &\leq \br{1 - \gamma \mu}^K \sqn{x_0 - x_\ast} + 2 \gamma^2 \sigmaesc^2 \br{ \sum_{k=0}^{K-1} \br{1 - \gamma \mu}^k } \\
		&\leq \br{1 - \gamma \mu}^K \sqn{x_0 - x_\ast} + 2 \gamma^2 \sigmaesc^2 \br{ \sum_{k=0}^{\infty} \br{1 - \gamma \mu}^k } \\
		&= \br{1 - \gamma \mu}^K \sqn{x_0 - x_\ast} + \frac{2 \gamma \sigmaesc^2}{\mu}.
		\newqedhere
	\end{align*}
\end{proof}

\section{Proofs for decreasing stepsize}

We first state and prove the following algorithm-independent lemma. This lemma plays a key role in the proof of \Cref{thm:psi-strongly-cvx-dec-stepsizes} and is heavily inspired by the stepsize schemes of \citet{Stich2019b} and \citet{ES-SGD-nonconvex} and their proofs. 

\begin{Lemma}
	\label{lemma:decreasing-stepsizes-recursion-solution}
	Suppose that there exist constants $a, b, c \geq 0$ such that for all $\gamma_t \leq \frac{1}{b}$ we have
	\begin{equation} 
		\label{eq:dec-stepsizes-recursion-init}
		\br{1 + \gamma_t a n} r_{t+1} \leq r_{t} + \gamma_t^3 c.
	\end{equation}
	Fix $T > 0$. Let $t_0 = \ceil{\frac{T}{2}}$. Then choosing stepsizes $\gamma_t > 0$ by
	\[
		\gamma_{t} = 
		\begin{cases}
			\frac{1}{b}, & \text { if } t \leq t_0 \text { or } T \leq \frac{b}{a n}, \\
			\frac{7}{a n \br{s + t - t_0}} & \text{ if } t > t_0 \text { and } T > \frac{b}{a n},
		\end{cases}	
	\]
	where $s = \frac{7b}{2 an}$. Then
	\[ r_{T} \leq \exp\br{-\frac{n T}{2 \br{b/a + n}}} r_0 + \frac{1421 c}{a^3 n^3 T^2}. \]
\end{Lemma}
\begin{proof}
	If $T \leq \frac{7b}{an}$, then we have $\gamma_t = \gamma = \frac{1}{b}$ for all $t$. Hence recursing we have,
	\begin{align*}
		r_{T} &\leq \br{1 + \gamma a n}^{-T} r_0 + \frac{\gamma^3 c}{\gamma a n} = \br{1 + \gamma a n}^{-T} r_0 + \frac{\gamma^2 c}{a n}.
	\end{align*}
	Note that $\frac{1}{1+x} \leq \exp(-\frac{x}{1+x})$ for all $x$, hence
	\begin{align*}
		r_{T} &\leq \exp\br{ - \frac{\gamma a n T}{1 + \gamma a n} } r_{0} + \frac{\gamma^2 c}{an}
	\end{align*}
	Substituting for $\gamma$ yields
	\begin{align*}
		r_{T} &\leq \exp\br{-\frac{nT}{b/a + n}} r_{0} + \frac{c}{b^2 a n}.
	\end{align*}
	Note that by assumption we have $\frac{1}{b} \leq \frac{7}{T a n}$, hence
	\begin{align}
		\label{eq:dec-step-1}
		r_{T} &\leq \exp\br{-\frac{nT}{b/a + n}} r_{0} + \frac{49 c}{T^2 a^3 n^3}.
	\end{align}
	If $T > \frac{7 b}{an}$, then we have for the first phase when $t \leq t_0$ with stepsize $\gamma_t = \frac{1}{b}$ that
	\begin{align}
		\label{eq:dec-step-2}
		r_{t_0} &\leq \exp\br{-\frac{nt_0}{b/a + n}} r_0 + \frac{c}{b^2 a n} \leq \exp\br{-\frac{nT}{2(b/a + n)}} r_{0} + \frac{c}{b^2 a n}.
	\end{align}
	Then for $t > t_0$ we have
	\[
		\br{1 + \gamma_t a n} r_{t+1} \leq r_{t} + \gamma_t^3 c = r_{t} + \frac{7^3 c}{a^3 n^3 \br{s + t - t_0}^3}.
	\]
	Multiplying both sides by $(s+t - t_0)^3$ yields 
	\begin{equation}
		\label{eq:dec-step-3}
		\br{s + t - t_0}^3 \br{1 + \gamma_t a n} r_{t+1} \leq \br{s + t - t_0}^3 r_{t} + \frac{7^3 c}{a^3 n^3}.
	\end{equation}
	Note that because $t$ and $t_0$ are integers and $t > t_0$, we have that $t - t_0 \geq 1$ and therefore $s + t - t_0 \geq 1$. We may use this to lower bound the multiplicative factor in the left hand side of \eqref{eq:dec-step-3} as
	\begin{align}
		\br{s + t - t_0}^3 \br{1 + \gamma_t a n} &= \br{s + t - t_0}^3 \br{1 + \frac{7}{s + t - t_0}} \nonumber \\
		&= \br{s + t - t_0}^3 + 7 \br{s + t - t_0}^2 \nonumber \\
		&= \br{s + t - t_0}^3 + 3 \br{s + t - t_0}^2 + 3 \br{s + t - t_0}^2 + \br{s + t - t_0}^2 \nonumber \\
		&\geq \br{s + t - t_0}^3 + 3 \br{s + t - t_0}^2 + 3 \br{s + t - t_0} + 1 \nonumber \\
		\label{eq:dec-step-4}
		&= \br{s + t + 1 - t_0}^3.
	\end{align}
	Using \eqref{eq:dec-step-4} in \eqref{eq:dec-step-3} we obtain 
	\[
		\br{s + t + 1 - t_0}^3 r_{t+1} \leq \br{s + t - t_0}^3 r_t + \frac{7^3 c}{a^3 n^3}.
	\]
	Let $w_{t} = \br{s + t - t_0}^3$. Then we can rewrite the last inequality as
	\[
		w_{t+1} r_{t+1} - w_t r_t \leq \frac{7^3 c}{a^3 n^3}.
	\]
	Summing up and telescoping from $t=t_0$ to $T$ yields
	\[
		w_{T} r_{T} \leq w_{t_0} r_{t_0} + \frac{7^3 c}{a^3 n^3} \br{T - t_0}.
	\]
	Note that $w_{t_0} = s^3$ and $w_{T} = \br{s + T - t_0}^3$. Hence,
	\begin{align*}
		r_{T} &\leq \frac{s^3}{\br{s + T - t_0}^3} r_{t_0} + \frac{7^3 c}{a^3 n^3 \br{s + T - t_0}^2} \frac{T - t_0}{s + T - t_0} \\
		&\leq \frac{s^3}{\br{s + T - t_0}^3} r_{t_0} + \frac{7^3 c}{a^3 n^3 \br{s + T - t_0}^2}.
	\end{align*}
	Since we have $s + T - t_0 \geq T - t_0 \geq T/2$, it holds
	\begin{equation}
		\label{eq:dec-step-5}
		r_{T} \leq \frac{8 s^3}{T^3} r_{t_0} + \frac{4 \cdot 7^3 c}{a^3 n^3 T^2}.
	\end{equation}
	The bound in \eqref{eq:dec-step-2} can be rewritten as
	\[
		\frac{s^3}{T^3} r_{t_0} \leq \frac{s^3}{T^3} \exp\br{-\frac{n T}{2 \br{b/a + n}}} r_{0} + \frac{s^3 c}{b^2 a n T^3}.
	\]
	We now rewrite the last inequality, use that $T > 2s$ and further use the fact that $s = \frac{7b}{2an}$:
	\begin{align}
		\frac{s^3}{T^3} r_{t_0} &\leq \underbrace{\br{\frac{s}{T}}^3}_{\leq 1/8} \exp\br{-\frac{n T}{2 \br{b/a + n}}} r_{0} + \frac{s^2 c}{b^2 a n T^2} \underbrace{\br{\frac{s}{T}}}_{\leq 1/2} \nonumber \\
		&\leq \frac{1}{8} \exp\br{-\frac{n T}{2 \br{b/a + n}}} r_{0} + \frac{s^2 c}{2 b^2 a n T^2} \nonumber \\
		\label{eq:dec-step-6}
		&= \frac{1}{8} \exp\br{-\frac{n T}{2 \br{b/a + n}}} r_{0} + \frac{7^2 c}{8 a^3 n^3 T^2}.
	\end{align}
	Plugging in the estimate of \eqref{eq:dec-step-6} into \eqref{eq:dec-step-5} we obtain
	\begin{align}
		r_{T} &\leq \exp\br{-\frac{n T}{2 \br{b/a + n}}} r_0 + \frac{7^2 c}{a^3 n^3 T^2} + \frac{4 \cdot 7^3 c}{a^3 n^3 T^2} \nonumber \\
		\label{eq:dec-step-7}
		&= \exp\br{-\frac{n T}{2 \br{b/a + n}}} r_0 + \frac{1421 c}{a^3 n^3 T^2}.
	\end{align}
	Taking the maximum of \eqref{eq:dec-step-1} and \eqref{eq:dec-step-7} we see that for any $T > 0$ we have
	\[
		r_{T} \leq \exp\br{-\frac{n T}{2 \br{b/a + n}}} r_0 + \frac{1421 c}{a^3 n^3 T^2}.
		\newqedhere
	\]
\end{proof}

\subsection{Proof of \Cref{thm:psi-strongly-cvx-dec-stepsizes}}
\begin{proof}
	Start with Lemma~\ref{lemma:inner-loop} with $\lambda = 0$, $L=L_{\max}$, and $\gamma=\gamma_t$,
	\[ \ecn{x_t^{i+1} - x_\ast^{i+1}} \leq \ecn{x_t^i - x_\ast^i} - 2 \gamma \br{1 - \gamma L_{\max}} \ec{D_{f_{\pi_i}}(x_t^i, x_\ast) } + 2 \gamma_t^3 \sigmarr^2. \]
	Since $\gamma \leq \frac{1}{L_{\max}}$ and $D_{f_{\pi}} (x_t^i, x_\ast)$ is nonnegative we may simplify this to
	\[ \ecn{x_t^{i+1} - x_\ast^{i+1}} \leq \ecn{x_t^i - x_\ast^i} + 2 \gamma_t^3 \sigmarr^2. \]
	Unrolling this recursion for $n$ steps we get
	\[ \ecn{x_t^{n} - x_\ast^n} \leq \ecn{x_t^0 - x_\ast^0} + 2 n \gamma_t^3 \sigmarr^2. \]
	By Lemma~\ref{prop:prox-contraction} and a similar reasoning to Theorem~\ref{thm:psi-strongly-convex-f-convex} we have
	\[ \br{1 + 2 \gamma_t \mu n} \ecn{x_{t+1} - x_\ast} \leq \ecn{x_t - x_\ast} + 2 \gamma_t^3 \sigmarr^2. \]
	We may then use Lemma~\ref{lemma:decreasing-stepsizes-recursion-solution} to obtain that
	\begin{align*}
		\ecn{x_T - x_\ast} &\leq \exp\br{-\frac{n T}{2(L_{\max}/\mu + n)}} \sqn{x_0 - x_\ast} + \frac{356 \sigmarr^2}{\mu^3 n^2 T^2} \\
		&= \mathcal{O}\br{ \exp\br{-\frac{n T}{\kappa + 2n}} \sqn{x_0 - x_\ast} + \frac{\sigmarr^2}{\mu^3 n^2 T^2} }.
		\newqedhere
	\end{align*}
\end{proof}

\section{Proof of \Cref{thm:IS} for importance resampling}
\begin{proof}
	We show that $N\le 2n$ as the rest of the theorem's claim trivially follows from \Cref{thm:psi-strongly-convex-f-convex}. Firstly, note that for any number $a\in\mathbb{R}$ we have $\lceil a \rceil \le a+1$. Therefore,
	\[
		N
		= \sum_{i=1}^n \left\lceil\frac{L_i}{\Lave}\right \rceil
		\le \sum_{i=1}^n \left(\frac{L_i}{\Lave}+1\right)
		= n + \sum_{i=1}^n \frac{L_i}{\Lave}
		= 2n. \newqedhere
	\]
\end{proof}

\section{Proofs for federated learning}
\subsection{Lemma for the extended proximal operator}
\begin{Lemma}
    \label{lem:ext-proximal-operator}
    Let $\psi_C$ be the consensus constraint and $R$ be a closed convex proximable function. Suppose that $x_1, x_2, \ldots, x_M$ are all in $\R^d$. Then,
    \[ \prox_{\gamma (R + \psi_C)} (x_1, \ldots, x_M) = \prox_{\frac{\gamma}{M}R}(\overline x), \]
    where $\overline x = \frac{1}{M}\sum_{m=1}^M x_m$.
\end{Lemma}
\begin{proof}
    We have,
    \[
        \prox_{\gamma(R+\psi_C)}(x_1, \dotsc, x_M)
        = \begin{pmatrix}
            \prox_{\frac{\gamma}{M}R}(\overline x)\\
            \vdots \\
            \prox_{\frac{\gamma}{M}R}(\overline x)
        \end{pmatrix}
        \quad \text{with} \quad \overline x = \frac{1}{M}\sum_{m=1}^M x_m.
    \]
    This is a simple consequence of the definition of the proximal operator. Indeed, the result of $\prox_{\gamma (R+\psi_C)}$ must have blocks equal to some vector $z$ such that
    \begin{align*}
        z
        &= \argmin_x \left\{\gamma R(x)+ \frac{1}{2}\sum_{m=1}^M\|x-x_m\|^2 \right\}\\
        &= \argmin_x \left\{\gamma R(x)+ \frac{1}{2}\sum_{m=1}^M \bigl(\|x-\overline x\|^2 + 2\<x-\overline x, \overline x - x_m>) + \|\overline x - x_m\|^2 \bigr)\right\}\\
        &= \argmin_x \left\{\gamma R(x)+ \frac{1}{2}M\|x-\overline x\|^2 \right\}
        \quad = \prox_{\frac{\gamma}{M}R}(\overline x).
    \end{align*}
\end{proof}

\subsection{Proof of \Cref{lem:fed_reform_properties}}
\begin{proof}
	Given some vectors $\xx, \yy\in\R^{d\cdot M}$, let us use their block representation $\xx=(x_1^\top,\dotsc, x_M^\top)^\top$, $\yy=(y_1^\top,\dotsc, y_M^\top)^\top$. Since we use the Euclidean norm, we have
	\[
		\|\nabla f_i(\xx)- \nabla f_i(\yy)\|^2
		= \sum_{m=1}^M \|\nabla f_{mi}(x_m)-\nabla f_{mi}(y_m)\|^2
		\le \sum_{m=1}^M L_i^2\|x_m-y_m\|^2
		= L_i^2\|\xx - \yy\|^2.
	\]
	We can obtain a lower bound by doing the same derivation and applying strong convexity instead of smoothness:
	\[
		\sum_{m=1}^M\|\nabla f_{mi}(x_m)-\nabla f_{mi}(y_m)\|^2
		\ge \mu^2 \sum_{m=1}^M\|x_m - y_m\|^2 
		= \mu^2 \|\xx-\yy\|^2.
	\]
	Thus, we have $\mu\|\xx-\yy\|\le \|\nabla f_i(\xx)- \nabla f_i(\yy)\|\le L_i\|\xx-\yy\|$, which is exactly $\mu$-strong convexity and $L_i$-smoothness of $f_i$.
\end{proof}

\subsection{Proof of \Cref{lem:fed_sigma}}
\begin{proof}
	By \Cref{thm:shuffling-radius-bound} we have
	\[
		\sigmarr^2 \le \frac{L_{\max}}{2}\Bigl(n^2\|\nabla f(\xx_\ast)\|^2 + \frac{n}{2}\sigmaesc^2\Bigr).
	\]
	Due to the separable structure of $f$, we have for the variance term
	\[
		n\sigmaesc^2 
		\eqdef  \sum_{i=1}^{n} \sqn{\nabla f_{i} (\xx_\ast) - \nabla f(\xx_\ast)}
		= \sum_{i=1}^{n}\sum_{m=1}^M \sqn{\nabla f_{mi} (x_\ast) - \frac{1}{n}\nabla F_m(x_\ast)}.
	\]
	The expression inside the summation is not exactly the variance due to the different normalization: $\frac{1}{n}$ instead of $\frac{1}{N_m}$. Nevertheless, we can expand the norm and try to get the actual variance:
	\begin{align*}
	     \sum_{i=1}^{n}\sqn{\nabla f_{mi} (x_\ast) - \frac{1}{n}\nabla F_m(x_\ast)} 
	     &=  \sum_{i=1}^{N_m}\biggl( \sqn{\nabla f_{mi} (x_\ast) - \frac{1}{N_m}\nabla F_m(x_\ast)} +\Bigl(\frac{1}{N_m} - \frac{1}{n}\Bigr)^2\sqn{\nabla F_m(x_\ast)}   \biggr)\\
	     &\quad + 2\sum_{i=1}^{N_m}\bigl\langle\nabla f_{mi}(x_\ast) - \frac{1}{N_m}\nabla F_m(x_\ast), \Bigl(\frac{1}{N_m} - \frac{1}{n}\Bigr)\nabla F_m(x_\ast) \bigr\rangle \\
	     &=  N_m\sigma_{m, \ast}^2 + N_m\Bigl(\frac{1}{N_m} - \frac{1}{n}\Bigr)^2\sqn{\nabla F_m(x_\ast)} \\
	     &\le n\sigma_{m, \ast}^2 + \sqn{\nabla F_m(x_\ast)}.
	\end{align*}
	Moreover, the gradient term has the same block structure, so
	\[
		n^2\|\nabla f(\xx_\ast)\|^2
		= n^2\biggl\|\frac{1}{n}\sum_{i=1}^n \nabla f_i(\xx_\ast) \biggr\|^2
		= \sum_{m=1}^M \sqn{\sum_{i=1}^n \nabla f_{mi}(x_\ast)}
		= \sum_{m=1}^M \|\nabla F_m(x_\ast)\|^2.
	\]
	Plugging the last two bounds back inside the upper bound on $\sigmarr^2$, we deduce the lemma's statement.
\end{proof}

\subsection{Proof of \Cref{thm:fed_hetero}}
\begin{proof}
	Since we assume that $N_1=\dotsb=N_M=n$, we have $\frac{N}{M}=n$ and the strong convexity constant  of$\psi=\frac{N}{n}(R+\psi_C)$ is equal to $\frac{N}{n}\cdot \frac{\mu}{M}=\mu$. By applying \Cref{thm:psi-strongly-convex-f-convex} we obtain
	\[
		\ecn{\xx_T - \xx_\ast} \leq \br{1 + 2\gamma \mu n}^{-T} \sqn{\xx_0 - \xx_\ast} + \frac{ \gamma^2 \sigmarr^2}{\mu}.
	\]
	Since $\xx_T = \prox_{\gamma N(R+\psi_C)}(\xx_{T-1}^n)$, we have $\xx_T \in C$, i.e., all of its blocks are equal to each other and we have $\xx_T=(x_T^\top,\dotsc, x_T^\top)^\top$. Since we use the Euclidean norm, it also implies
	\[
		\ecn{\xx_T - \xx_\ast}
		= M \|x_T - x_\ast\|^2.
	\]
	The same is true for $\xx_0$, so we need to divide both sides of  the upper bound on $\|\xx_T -\xx_\ast\|^2$ by $M$. Doing so together with applying \Cref{lem:fed_sigma} yields
	\begin{align*}
		\ecn{x_T - x_\ast} 
		&\leq \br{1 + 2\gamma \mu n}^{-T} \sqn{x_0 - x_\ast}  + \frac{ \gamma^2 \sigmarr^2}{M\mu} \\
		&\le \br{1 + 2\gamma \mu n}^{-T} \sqn{x_0 - x_\ast}  + \frac{ \gamma^2 L_{\max} }{M\mu} \sum_{m=1}^M\Bigl( \|\nabla F_m(x_\ast)\|^2 + \frac{n}{4}\sigma_{m, \ast}^2\Bigr) \\
		&= \br{1 + 2\gamma \mu n}^{-T} \sqn{x_0 - x_\ast}  + \frac{ \gamma^2 L_{\max}}{M\mu} \sum_{m=1}^M\Bigl(  \|\nabla F_m(x_\ast)\|^2 + \frac{N}{4M}\sigma_{m, \ast}^2\Bigr).
	\end{align*}
\end{proof}

\subsection{Proof of \Cref{thm:fed_iid}}
\begin{proof}
	According to~\Cref{lem:fed_reform_properties}, each $f_i$ is $\mu$-strongly convex and $L_{\max}$-smooth, so we obtain the result by trivially applying \Cref{thm:f-strongly-convex-psi-convex} and upper bounding $\sigmarr^2$ the same way as in the proof of~\Cref{thm:fed_hetero}.
\end{proof}

\section{Federated experiments and experimental details}
We also compare the performance of FedRR and Local SGD on homogeneous (i.e., i.i.d.) data. Since Local SGD requires smaller stepsizes to converge, it is significantly slower at initialization, as can be seen in Figure~\ref{fig:fed_rr}. FedRR, however, does not need small initial stepsize and very quickly converges to a noisy neighborhood of the solution. The advantage is clear both from the perspective of the number of communication rounds and data passes.

To illustrate the severe impact of the number of local steps in Local SGD we show results with different number of local steps. The blue line shows Local SGD that takes the number of steps equivalent to full pass over the data by each node. The orange line takes 5 times fewer local steps. Clearly, the latter performs better in terms of communication rounds and local steps, making it clear that Local SGD scales worse with the number of local steps. This phenomenon is well-understood and has been in discussed by~\citet{khaled2020tighter}.
\begin{figure}[t]
\centering
	\includegraphics[scale=0.20]{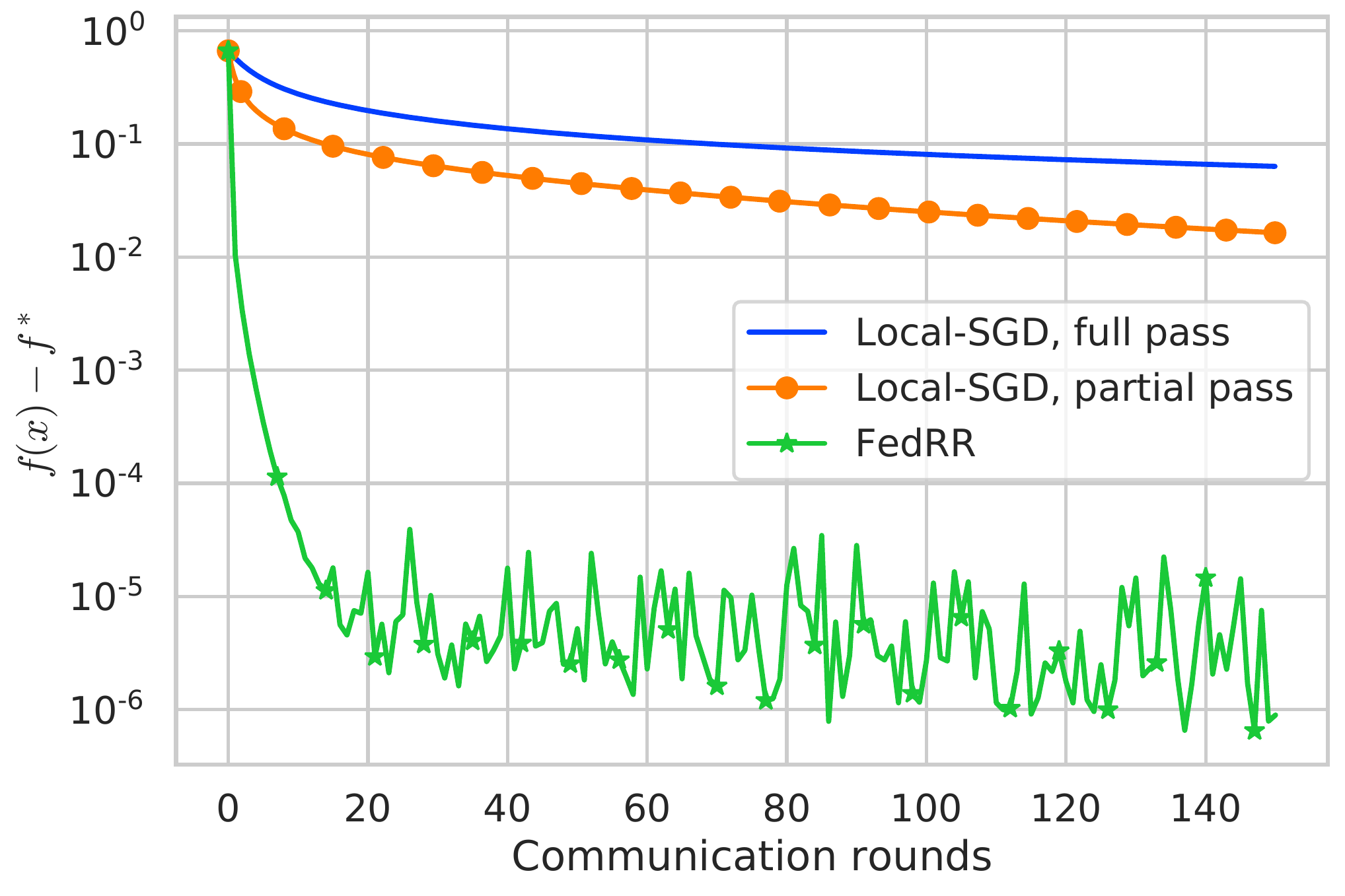}
	\includegraphics[scale=0.20]{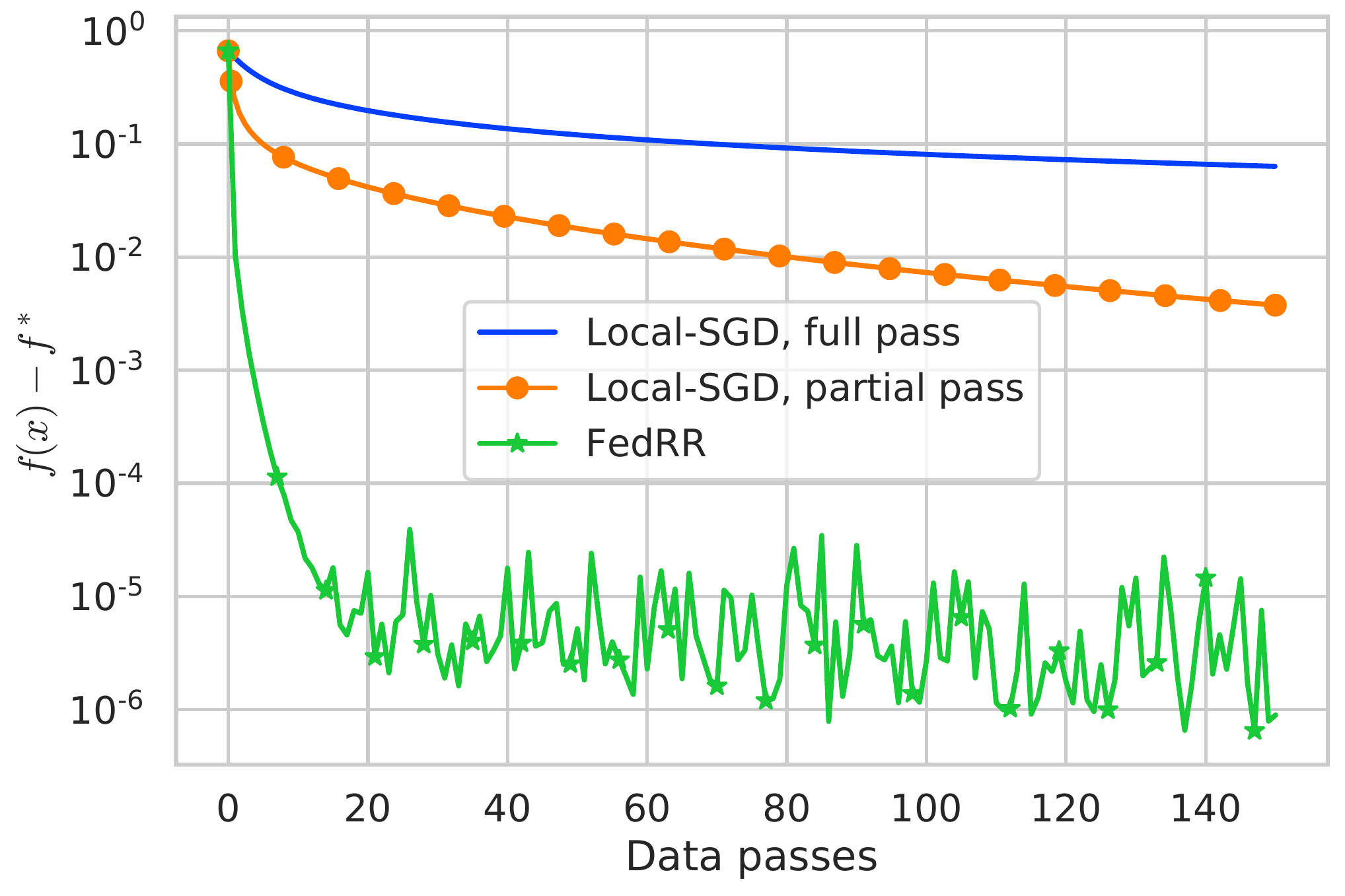}
\caption{Experimental results for parallel training. Left: comparison in terms of communication rounds, right: in terms of data passes}
\label{fig:fed_rr}
\end{figure}

\textbf{Implementation details.} For each $i$, we have $L_i=\frac{1}{4}\|a_i\|$. We set $\lambda_2=\frac{L}{N}$ and tune $\lambda_1$ to obtain a solution with less than 50\% coordinates (exact values are provided in the code). We use stepsizes decreasing as $\cO(\frac{1}{t})$ for all methods.  We use the `a1a' dataset for the experiment with $\ell_1$ regularization.

The experiment for the comparison of FedRR and Local SGD uses no $\ell_1$ regularization and $\lambda_2=\frac{L}{N}$. We choose the stepsizes according to the theory of Local SGD and Fed-RR. As per Theorem~3 in~\cite{khaled2020tighter}, the stepsizes for Local SGD must satisfy $\gamma_t = \cO(1 / (LH))$, where $H$ is the number of local steps.  The parallelization of local runs is done using the Ray package\footnote{\href{https://ray.io/}{https://ray.io/}}. We use the `mushrooms' dataset for this experiment.

\textbf{Proximal operator calculation.} As shown by \cite{Neal2014}, the proximal operator for $\psi(x)=\lambda_1\|x\|_1 + \frac{\lambda_2}{2}\|x\|^2$ is given by
\[
	\prox_{\gamma\psi}(x) = \frac{1}{1+\gamma\lambda_2}\prox_{\gamma\lambda_1\|\cdot\|_1}(x),
\]
where the $j$-th coordinate of $\prox_{\gamma\lambda_1\|\cdot\|_1}(x)$ is
\[
	[\prox_{\gamma\lambda_1\|\cdot\|_1}(x)]_j
	= \begin{cases}
		\mathrm{sign}([x]_j)(|[x]_j|-\gamma\lambda_1), & \text{if } |[x]_j|\ge \gamma\lambda_1,\\
		0, &\text{otherwise}.
	\end{cases}
\]
\end{document}